\documentclass[10pt,letterpaper]{article}
\usepackage[dvipsnames]{xcolor}
\usepackage{ragged2e}

\usepackage{soul}
\usepackage{graphicx}
\usepackage{color}
\usepackage[shortlabels]{enumitem}
\usepackage{tikz}
\usetikzlibrary{positioning}
\usetikzlibrary{arrows}
\usetikzlibrary{patterns,calc}

\usepackage[top=0.85in,left=2.75in,footskip=0.75in]{geometry}

\usepackage{amsmath,amssymb}
\usepackage{graphicx}

\usepackage{changepage}

\usepackage[utf8x]{inputenc}

\usepackage{textcomp,marvosym}

\usepackage{cite}

\usepackage{nameref,hyperref}

\usepackage[right]{lineno}

\usepackage{microtype}
\DisableLigatures[f]{encoding = *, family = * }

\usepackage{array}

\newcolumntype{+}{!{\vrule width 2pt}}

\newlength\savedwidth

\raggedright
\usepackage[aboveskip=1pt,labelfont=bf,labelsep=period,justification=raggedright,singlelinecheck=off]{caption}

\makeatletter
\renewcommand{\@biblabel}[1]{\quad#1.}
\makeatother

\usepackage{lastpage,fancyhdr,graphicx}
\usepackage{epstopdf}
\fancyheadoffset[L]{2.25in}
\fancyfootoffset[L]{2.25in}
\usepackage{float}
\usepackage{stackengine}
\usepackage{scalerel}
\usepackage{color}
\usepackage{algorithm}
\usepackage[noend]{algpseudocode}
\usepackage{enumitem}
\usepackage[export]{adjustbox}
\usepackage{subcaption}

\newcommand{\avsum}{\mathop{\mathpalette\avsuminner\relax}\displaylimits}
\makeatletter
\newcommand\avsuminner[2]{
  {\sbox0{$\m@th#1\sum$}
   \vphantom{\usebox0}
   \ooalign{
     \hidewidth
     \smash{\vrule height\dimexpr\ht0+1pt\relax depth\dimexpr\dp0+1pt\relax}%
     \hidewidth\cr
     $\m@th#1\sum$\cr
   }
  }
}

\newcommand{\BlackBox}{\rule{1.5ex}{1.5ex}}  
\newenvironment{proof}{\par\noindent{\bf Proof\ }}{\hfill\BlackBox\\[2mm]}
\newtheorem{theorem}{Theorem}

\newcommand{\mb}[1]{\mathbf{#1}}

\begin{document}
\justifying
\vspace*{0.2in}

\begin{flushleft}
{\Large
\textbf\newline{Introducing Neuromodulation in Deep Neural Networks to Learn Adaptive Behaviours} 
}
\newline
\\
Vecoven Nicolas\textsuperscript{1*},
Ernst Damien\textsuperscript{1},
Wehenkel Antoine\textsuperscript{1},
Drion Guillaume\textsuperscript{1}
\\
\bigskip
\textbf{1} Department of Electrical Engineering and Computer Science 
       Montefiore Institute, University of Li\`ege
       B-4000 Li\`ege, Belgium
\\
\bigskip

* nvecoven@uliege.be

\end{flushleft}

\begin{abstract}
 
Animals excel at adapting their intentions, attention, and actions to the environment, making them remarkably efficient at interacting with a rich, unpredictable and ever-changing external world, a property that intelligent machines currently lack. Such an adaptation property relies heavily on \emph{cellular neuromodulation}, the biological mechanism that dynamically controls intrinsic properties of neurons and their response to external stimuli in a context-dependent manner. In this paper, we take inspiration from cellular neuromodulation to construct a new deep neural network architecture that is specifically designed to learn adaptive behaviours. The network adaptation capabilities are tested on navigation benchmarks in a meta-reinforcement learning context and compared with state-of-the-art approaches. Results show that neuromodulation is capable of adapting an agent to different tasks and that neuromodulation-based approaches provide a promising way of improving adaptation of artificial systems.  


\end{abstract}

\section{Introduction}
The field of machine learning has seen tremendous progress made during the past decade, predominantly owing to the improvement of deep neural network (DNN) algorithms. DNNs are networks of artificial neurons whose interconnections are tuned to reach a specific goal through the use of an optimization algorithm, mimicking the role of synaptic plasticity in biological learning. This approach has led to the emergence of highly efficient algorithms that are capable of learning and solving complex problems. Despite these tremendous successes, it remains difficult to learn models that generalise or adapt themselves efficiently to new, unforeseen problems based on past experiences. This calls for the development of novel architectures specifically designed to enhance adaptation capabilities of current DNNs. 

In biological nervous systems, adaptation capabilities have long been linked to neuromodulation, a biological mechanism that acts in concert with synaptic plasticity to tune neural network functional properties. In particular, cellular neuromodulation provides the ability to continuously tune neuron input/output behaviour to shape their response to external stimuli in different contexts, generally in response to an external signal carried by biochemicals called neuromodulators \cite{bargmann2013connectome, marder2014neuromodulation}. Neuromodulation regulates many critical nervous system properties that cannot be achieved solely through synaptic plasticity \cite{marder1996principles, marder2001central}. It has been shown as being critical to the adaptive control of continuous behaviours, such as in motor control, among others \cite{marder1996principles, marder2001central}. In this paper, we introduce a new neural architecture specifically designed for DNNs and inspired from cellular neuromodulation, which we call NMN, standing for ``Neuro-Modulated Network''. 

At its core, the NMN architecture comprises two neural networks: a main network and a neuromodulatory network. The main network is a feed-forward DNN composed of neurons equipped with a parametric activation function whose parameters are the targets of neuromodulation. It allows the main network to be adapted to new unforeseen problems. The neuromodulatory network, on the other hand, dynamically controls the neuronal properties of the main network via the parameters of its activation functions. Both networks have different inputs: the neuromodulatory network processes feedback and contextual data whereas the main network is in charge of processing other inputs.

Our proposed architecture can be related to previous works on different aspects. In \cite{miconi2018differentiable}, the authors take inspiration from Hebbian plasticity to build networks with plastic weights, allowing them to tune their weights dynamically. In \cite{miconi2018backpropamine} the same authors extend their work by learning a neuromodulatory signal that dictates which and when connections should be plastic. Our architecture is also related to hypernetworks \cite{ha2016hypernetworks}, in which a network's weights are computed through another network. Finally, other recent works focused on learning fixed activation functions \cite{apl, lin2013network}.


\section{NMN architecture}
\label{sec:NMN}
The NMN architecture revolves around the neuromodulatory interaction between the neuromodulatory and main networks. We mimic biological cellular neuromodulation \cite{drion2015neuronal} in a DNN by assigning the neuromodulatory network the task to tune the slope and bias of the main network activation functions. 

Let $\sigma(x): \mathbb{R}\rightarrow\mathbb{R}$ denote any activation function and its neuromodulatory capable version $\sigma_{\text{NMN}}(x, \mb{z}; \mb{w}_s, \mb{w}_b) = \sigma\left( \mb{z}^T (x \mb{w}_s + \mb{w}_b)  \right) $ where $\mb{z} \in \mathbb{R}^k$ is a neuromodulatory signal and $\mb{w}_s, \mb{w}_b \in \mathbb{R}^k$ are two parameter vectors of the activation function, respectively governing a scale factor and an offset. In this work, we propose to replace all the main network neuron activation functions with their neuromodulatory capable counterparts. The neuromodulatory signal $\mb{z}$, where size $k$ is a free parameter, is shared for all these neurons and computed by the neuromodulatory network as $\mb{z} = f(\mb{c})$, where $\mb{c}$ is a vector representing contextual and feedback inputs. The function $f$ can be any DNN taking as input such vector $\mb{c}$. For instance, $\mb{c}$ may have a dynamic size (e.g. more information about the current task becomes available as time passes), in which case $f$ could be parameterised as a recurrent neural network (RNN) or a conditional neural process \cite{cnp}, enabling refinement of the neuromodulatory signal as more data becomes available. The complete NMN architecture and the change made to the activation functions are depicted in Figure \ref{fig:high_level_view}. 

\begin{figure}[htbp]
\centering
\definecolor{imblue}{HTML}{44BDDA}
\begin{tikzpicture}[scale=0.93, minimum size=1em]

\draw[white, line width=1pt] (-6,0) rectangle (7.75,4); 

\begin{scope}[xshift=-0.20cm]
\node[] at (-5.7,3.75) {\textbf{A}};
\node[inner sep=0pt] (main) at (-3.55,1) (main)
{\includegraphics[width=.15\textwidth]{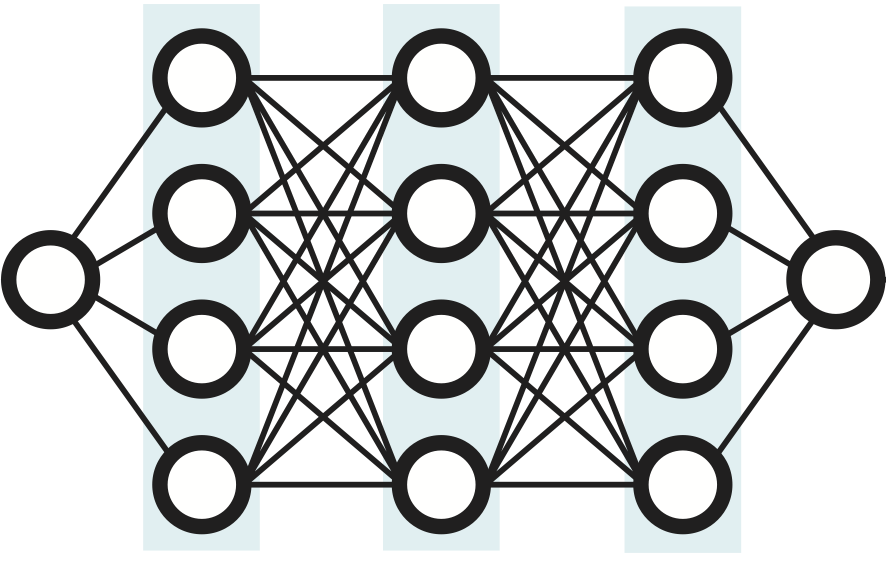}};
\node[inner sep=0pt] (neuromodulatory) at (-3.55,2.8)
{\includegraphics[width=.1\textwidth]{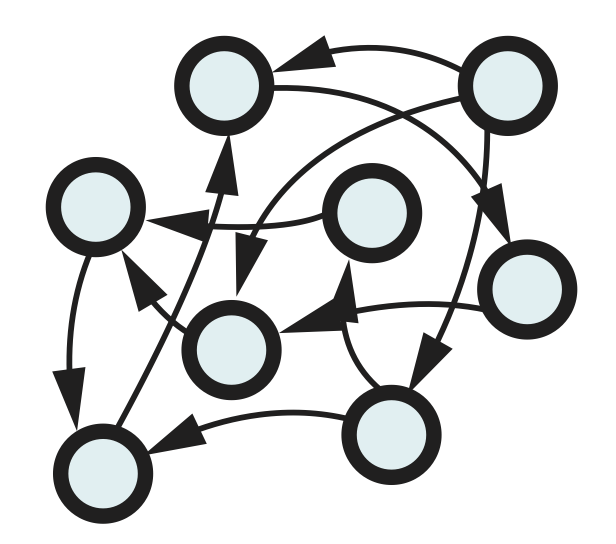}};
\draw[line width=1pt,OliveGreen,draw] (-4.4,2.1) rectangle (-2.75,3.5);
\draw[-, line width=0.5pt,imblue,fill=imblue] (-2.98,2.1) -- +(0,-0.35) circle (0.15em);
\draw[-, line width=0.5pt,imblue,fill=imblue] (-3.55,2.1) -- +(0,-0.35) circle (0.15em);
\draw[-, line width=0.5pt,imblue,fill=imblue] (-4.125,2.1) -- +(0,-0.35) circle (0.15em);
\draw[<-,line width=1pt] (main) -- +(-1.5,0) node[near end, left] {\small Inputs};
\draw[->,line width=1pt,shorten <=0.05cm,] (main) -- +(+1.5,0) node[near end, right,yshift=-0.1ex] {\small Outputs};
\draw[<-,line width=1pt,shorten >=0.1cm,shorten <=0.1cm,] (neuromodulatory) -- +(+1.5,0) node[near end, right, align=center,yshift=-1.1ex] {\small Context\\$\mathbf{c}$};
\node[] at (-3.5,3.75) {\small \textit{neuromodulatory}};
\node[] at (-3.55,0.2) {\small \textit{main}};
\end{scope}

\begin{scope}
\node[] at (-0.5,3.75) {\textbf{B}};	    
\draw[gray, rounded corners=2mm, line width=0.5pt] (0.5,0.25) rectangle (6.35,3.75); 
\node[minimum size=2.5em, circle, draw=black,line width=1pt] at (-0.25,1) (X1) {\footnotesize $x$};
\node[line width=1pt,minimum size=1.4em,draw,imblue] at (2.5,1) (mul) {$\times$};
\node[line width=1pt,draw,minimum size=1.4em,imblue] at (3.75,1) (plus) {$+$};
\node[line width=1pt,minimum size=1.4em, draw] at (5.5,1) (sig) {$\sigma$};
\node[line width=1pt,minimum size=2.5em,circle, draw=black] at (1.25,3) (ws) {\footnotesize $\boldsymbol{w}_s$};
\node[line width=1pt,draw,imblue,minimum size=1.4em] at (2.5,3) (dot1) {$\cdot$};
\node[line width=1pt,OliveGreen,minimum size=2.5em, circle,  draw] at (3.75,3) (z) {\footnotesize $\boldsymbol{z}$};
\node[line width=1pt,imblue, draw,minimum size=1.4em] at (3.75,2) (dot2) {$\cdot$};
\node[line width=1pt,minimum size=2.5em,,circle, draw=black] at (5.5,2) (wb) {\footnotesize $\boldsymbol{w}_b$};
\node[line width=1pt,minimum size=2.5em,circle, draw=black] at (7.25,1) (y) {\footnotesize $y$};
\node[] at (5.75,3.45) {\small $\sigma_{\text{NMN}}$};
\draw[->,line width=1pt] (X1) -- (mul);
\draw[->,line width=1pt] (mul) -- (plus);
\draw[->,line width=1pt] (plus) -- (sig);
\draw[->,line width=1pt] (sig) -- (y);
\draw[->,line width=1pt] (ws) -- (dot1);
\draw[->,line width=1pt] (z) -- (dot1);
\draw[->,line width=1pt] (dot1) -- (mul);
\draw[->,line width=1pt] (z) -- (dot2);
\draw[->,line width=1pt] (dot2) -- (plus);
\draw[->,line width=1pt] (wb) -- (dot2);
\end{scope}

\end{tikzpicture}
\caption{Sketch of the NMN architecture. \textbf{A.} The NMN is composed of the interaction of a \textit{neuromodulatory} neural network that processes some context signal (top) and a \textit{main} neural network that shapes some input-output function (bottom). \textbf{B.} Computation graph of the NMN activation functions $\sigma_{NMN}$, where $\mb{w}_s$ and $\mb{w}_b$ are parameters controlling the scale factor and the offset of the activation function $\sigma$, respectively. $\textbf{z}$ is a context-dependent variable computed by the neuromodulatory network.}
\label{fig:high_level_view}
\end{figure}
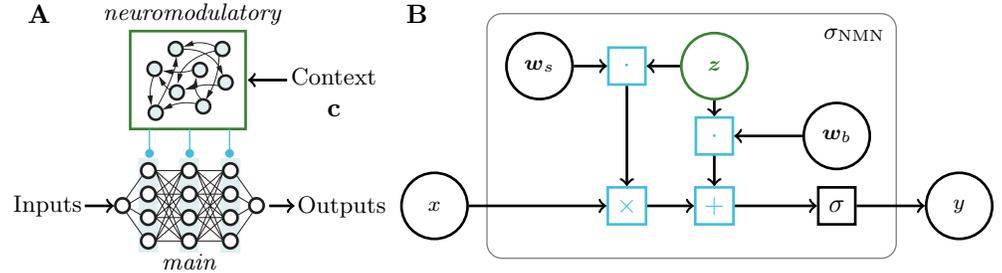

Notably, the number of newly introduced parameters scales linearly with the number of neurons in the main network whereas it would scale linearly with the number of connections between neurons if the neuromodulatory network was affecting connection weights, as seen for instance in the context of hypernetworks \cite{ha2016hypernetworks}. Therefore our approach can be extended more easily to very large networks. 

\section{Experiments}
\label{sec:experiments}

\subsection{Setting}
\label{subsec:formalization}
A good biologically motivated framework to which the NMN can be applied and evaluated is meta-reinforcement learning (meta-RL), as defined in \cite{learningtorl}. 
In contrast with classical reinforcement learning (RL), which is formalised as the interaction between an agent and an environment defined as a Markov decision process (MDP), the meta-RL setting resides in the sub-division of an MDP as a distribution $\mathcal{D}$ over simpler MDPs.
Let $t$ denote the discrete time, $\mb{x}_t$ the state of the MDP at time $t$, $\mb{a}_t$ the action taken at time $t$ and $r_t$ the reward obtained at the subsequent time-step. 
At the beginning of a new episode $i$, a new element is drawn from $\mathcal{D}$ to define an MDP, referred to as $\mathcal{M}$, with which the meta-RL agent interacts for $T \in \mathbb{N}$ time-steps afterwards. The only information that the agent collects on $\mathcal{M}$ is through observing the states crossed and the rewards obtained at each time-step.
We denote by $\mb{h}_{t}=\left[\mb{x}_{0},\mb{a}_{0},r_{0},\mb{x}_{1},\ldots,\mb{a}_{t-1}, r_{t-1}, \mb{x}_t\right]$ the history of the interaction with $\mathcal{M}$ up to time-step $t$. As in \cite{learningtorl}, the goal of the meta-learning agent is to maximise the expected value of the discounted sum of rewards it can obtain over all the time-steps and episodes.

\subsection{Training}
\label{subsec:training}

In \cite{learningtorl}, the authors tackle this meta-RL framework by using an advantage actor-critic (A2C) algorithm. This algorithm revolves around two distinct parametric functions: the actor and the critic. The actor represents the policy used to interact with the MDPs, while the critic is a function that rates the performance of the agent policy. All actor-critic algorithms follow an iterative procedure that consists of the three following steps. 
\begin{enumerate}
    \item Use the policy to interact with the environment and gather data.
    \item Update the actor parameters using the critic ratings.
    \item Update the critic parameters to better approximate a value function.
\end{enumerate}

In \cite{learningtorl}, the authors chose to model the actor and the critic with RNNs, taking $\mb{h}_{t}$ as the input. In this work, we propose comparing the NMN architecture to standard RNN by modelling both the actor and the critic with NMN. To this end, we define the feedback and contextual inputs $\mb{c}$ (i.e. the neuromodulatory network inputs) as $\mb{h}_{t} \setminus \mb{x}_t$ while the main network input is defined as $\mb{x}_{t}$. Note that $\mb{h}_{t}$ grows as the agent interacts with $\mathcal{M}$, motivating the usage of a RNN as neuromodulatory network. A graphical comparison between both architectures is shown on Figure~\ref{fig:tikz_model}.

\begin{figure}[htbp]
\centering
\definecolor{imblue}{HTML}{44BDDA}
\begin{tikzpicture}[scale=1, minimum size=1em]
\draw[white, line width=1pt] (-6,0) rectangle (7.75,4); 
\begin{scope}
\node[] at (-5.7,4.5) {\textbf{A}};

\coordinate (rnn) at (-4, 2.6);
\node[line width=1pt,minimum size=0.5em, draw=black] at (rnn) (node_rnn) {RNN};
\node[line width=1pt,minimum size=0.5em, draw=black] at ($(rnn)+(2.0,0.0)$) (node_mlp) {MLP};
\node[line width = 0pt] at ($(rnn)+(0.0,1.0)$) (input_rnn_1){$[\mb{x}_t, \mb{a}_{t-1}, r_{t-1}]$};
\node[line width = 0pt] at ($(node_mlp)+(1.4, 0.0)$) (output_rnn_1){$\mb{a}_{t}$};
\draw[->] (input_rnn_1)--(node_rnn);
\draw[->] (node_rnn)--(node_mlp);
\draw[->] (node_mlp)--(output_rnn_1);

\coordinate (rnn2) at (-4, 0.4);
\node[line width=1pt,minimum size=0.5em, draw=black] at (rnn2) (node_rnn2) {RNN};
\node[line width=1pt,minimum size=0.5em, draw=black] at ($(rnn2)+(2.0,0.0)$) (node_mlp2) {MLP};

\node[line width = 0pt] at ($(rnn2)+(0.0,1.0)$) (input_rnn_2){$[\mb{x}_{t+1}, \mb{a}_{t}, r_{t}]$};

\node[line width = 0pt] at ($(node_mlp2)+(1.4, 0.0)$) (output_rnn_2){$\mb{a}_{t+1}$};
\draw[->] (input_rnn_2)--(node_rnn2);
\draw[->] (node_rnn2)--(node_mlp2);

\draw[->] (node_mlp2)--(output_rnn_2);

\draw[->, dashed, line width = 1.5] (node_rnn) to [out=-180,in=-180](node_rnn2) node[midway, right]{};

\end{scope}

\begin{scope}
\node[] at (1.0,4.5) {\textbf{B}};

\coordinate (main) at (5.3, 2.6);
\coordinate (nmd1) at ($(main) + (-2.0, 0.0)$);
\node[line width=1pt,minimum size=0.5em, draw=black] at (main) (node_main1) {MLP};
\node[line width=1pt,minimum size=0.5em, draw=black] at (nmd1) (node_nmd1) {RNN};
\node[line width = 0pt] at ($(main)+(0.0,1.0)$) (input_main_1){$\mb{x}_t$};
\node[line width = 0pt] at ($(nmd1)+(0.0,1.0)$) (input_nmd_1){$[\mb{x}_{t-1},\mb{a}_{t-1},r_{t-1}]$};
\node[line width = 0pt] at ($(main)+(1.4,0.0)$) (output_rnn_1){$\mb{a}_{t}$};
\draw[->] (input_main_1)--(node_main1);
\draw[->] (input_nmd_1)--(node_nmd1);
\draw[-o] (node_nmd1)--(node_main1);
\draw[->] (node_main1)--(output_rnn_1);
\coordinate (main1) at (main);
\coordinate (nmdo) at (nmd1);

\coordinate (main) at (5.3, 0.4);

\coordinate (nmd1) at ($(main) + (-2.0, 0.0)$);
\node[line width=1pt,minimum size=0.5em, draw=black] at (main) (node_main1) {MLP};
\node[line width=1pt,minimum size=0.5em, draw=black] at (nmd1) (node_nmd2) {RNN};
\node[line width = 0pt] at ($(main)+(0.0,1.0)$) (input_main_1){$\mb{x}_{t+1}$};
\node[line width = 0pt] at ($(nmd1)+(0.0,1.0)$) (input_nmd_1){$[\mb{x}_{t},\mb{a}_{t},r_{t}]$};
\node[line width = 0pt] at ($(main)+(1.4,0.0)$) (output_rnn_1){$\mb{a}_{t+1}$};
\draw[->] (input_main_1)--(node_main1);
\draw[->] (input_nmd_1)--(node_nmd2);
\draw[-o] (node_nmd2)--(node_main1);
\draw[->] (node_main1)--(output_rnn_1);

\draw[->, dashed, line width = 1.5] (node_nmd1) to [out=-180,in=-180](node_nmd2) node[midway, right]{};

\draw[line width=1pt,black,draw, opacity=0.15, rounded corners] ($(rnn)+(-0.75,-0.5)$) rectangle ($(node_mlp)+(0.75,0.5)$);
\draw[line width=1pt,black,draw, opacity=0.15, rounded corners] ($(rnn2)+(-0.75,-0.5)$) rectangle ($(node_mlp2)+(0.75,0.5)$);

\draw[line width=1pt,black,draw, opacity=0.15, rounded corners] ($(nmdo)+(-0.75,-0.5)$) rectangle ($(main1)+(0.75,0.5)$);
\draw[line width=1pt,black,draw, opacity=0.15, rounded corners] ($(nmd1)+(-0.75,-0.5)$) rectangle ($(main)+(0.75,0.5)$);

\end{scope}

\end{tikzpicture}
\caption{Sketch of a standard recurrent network (\textbf{A}) and of an NMN (\textbf{B}) in a meta-RL framework. $\rightarrow$ represent standard connections, $\multimap$ represent a neuromodulatory connection, $\dashrightarrow$ represent temporal connections and $MLP$ stands for Multi-Layer Perceptron (standard feed-forward network).}
\label{fig:tikz_model}
\end{figure}
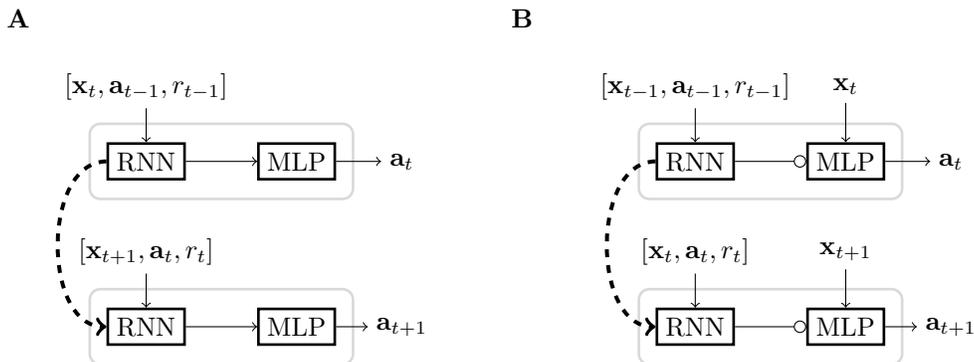

To be as similar as possible to the neuronal model proposed by \cite{drion2015neuronal}, the main network is a fully-connected neural network built using saturated rectified linear unit (sReLU) activation functions $\sigma(x) = \min(1,\max(-1, x))$, except for the final layer (also neuromodulated), for which $\sigma(x) = x$. In Section~\ref{sec:results}, we also report results obtained with sigmoidal activation functions which are often appreciably inferior to those obtained with sReLUs, further encouraging their use.

We built our models such that both standard RNN and NMN architectures have the same number of recurrent layers/units and a relative difference between the numbers of parameters that is lower than $2\%$. Both models are trained using an A2C algorithm with generalized advantage estimation \cite{GAE} and proximal policy updates \cite{PPO}. Finally, no parameter is shared between the actor and the critic. We motivate this choice by noting that the neuromodulatory signal might need to be different for the actor and the critic. For completeness and reproducibility, we provide a formal description of the algorithms used as supplementary material (\ref{app:RL}). This material aims mainly to describe and discuss standard RL algorithms in the context of meta-RL and, to a lesser extent, it aims to provide full implementation details. We also provide the exact neural architectures used for each benchmark as supplementary material (\ref{app:arch}).

\subsection{Benchmarks description}
\label{subsec:benchmarks}

We carried out our experiments on three custom benchmarks: a simple toy problem and two navigation problems with sparse rewards. These benchmarks were built to evaluate our architecture in environments with continuous action spaces. 
For conciseness and clarity, we only provide a mathematical definition of the first benchmark. The two other benchmarks are briefly textually depicted and further details are available as supplementary material (see \ref{app:problems}). Figures~\ref{fig:pres_bench1},~\ref{fig:pres_bench2}~and~\ref{fig:pres_bench3} are a graphical representation of each of the benchmarks.

\begin{figure}[htbp]
\centering
\definecolor{imblue}{HTML}{44BDDA}
\begin{tikzpicture}[scale=1, minimum size=1em]
\draw[white, line width=1pt] (-6,0) rectangle (7.75,4); 
\begin{scope}
\node[] at (-5.7,3.5) {\textbf{A}};
\coordinate (bar) at (-2.7, 2.0);
\coordinate (bar_l) at ($(bar)+(-2.7,-0.1)$);
\coordinate (bar_r) at ($(bar)+(2.7,0.1)$);
\coordinate (alpha) at (1.4, 0.1);
\draw[->, line width = 0.8pt, opacity=.8, black] ($(bar)+(alpha)+(0.0,0.7)$) -- ($(bar)+(alpha)$);
\draw[dashed, line width = 0.8pt, opacity=.6, black] ($(bar)+(alpha)+(0.4,0.7)$) -- ($(bar)+(alpha)+(0.4,-0.0)$);
\draw[dashed, line width = 0.8pt, opacity=.6, black] ($(bar)+(alpha)+(-0.4,0.7)$) -- ($(bar)+(alpha)+(-0.4,0.0)$);
\coordinate (state) at (-2.3, 0.1);
\draw[->, line width = 0.8pt, opacity=.8, imblue] ($(bar)+(state)+(0.0,0.7)$) -- ($(bar)+(state)$);
\coordinate (action) at (-1.6, 0.1);
\draw[->, line width = 0.8pt, opacity=.8, OliveGreen] ($(bar)+(action)+(0.0,0.7)$) -- ($(bar)+(action)$);
\draw[-, line width = 0.8pt, opacity=.8] ($(bar)+(0.0,0.3)$) -- ($(bar)+(0.0,0.1)$);
\draw[-, dashed, line width = 0.8pt, opacity=.65] ($(bar)+(0.0,0.1)$) -- ($(bar)+(0.0,-1.3)$);

\node[] at ($(bar)+(0.0,0.5)$) {\footnotesize $0$};

\draw[->, color = imblue, line width = 1.5] ($(bar)+(0.0,-0.3)$) -- ($(state)+(bar)+(0.0,-0.4)$) node[midway, below]{\footnotesize \textcolor{imblue}{$x_{t+1} = x_{t}$}};
\draw[->, color = OliveGreen, line width = 1.5] ($(bar)+(0.0,-0.7)$) -- ($(action)+(bar)+(0.0,-0.8)$) node[midway, below]{\footnotesize \textcolor{OliveGreen}{$a_t$}};
\draw[->, color = black, line width = 1.5] ($(bar)+(0.0,-0.3)$) -- ($(bar)+(alpha)+(0.0,-0.4)$) node[midway, below]{\footnotesize \textcolor{black}{$p_{t+1} = p_{t}$}};
\draw[->, color = imblue, line width = 1.5] ($(bar)+(alpha)+(0.0,-1.2)$) -- ($(state)+(bar)+(0.0,-1.2)$) node[midway, below]{\footnotesize \textcolor{imblue}{$\alpha$}};

\draw[black, line width=0.5pt] (bar_l) rectangle (bar_r); 

\end{scope}

\begin{scope}
\node[] at (1.0,3.5) {\textbf{B}};	    
\coordinate (bar) at (4, 2.0);
\coordinate (bar_l) at ($(bar)+(-2.7,-0.1)$);
\coordinate (bar_r) at ($(bar)+(2.7,0.1)$);
\coordinate (alpha) at (-1.3, 0.1);
\draw[->, line width = 0.8pt, opacity=.8, black] ($(bar)+(alpha)+(0.0,0.7)$) -- ($(bar)+(alpha)$);
\draw[dashed, line width = 0.8pt, opacity=.6, black] ($(bar)+(alpha)+(0.4,0.7)$) -- ($(bar)+(alpha)+(0.4,-0.0)$);
\draw[dashed, line width = 0.8pt, opacity=.6, black] ($(bar)+(alpha)+(-0.4,0.7)$) -- ($(bar)+(alpha)+(-0.4,0.0)$);
\coordinate (state) at (1.7, 0.1);
\draw[->, line width = 0.8pt, opacity=.8, imblue] ($(bar)+(state)+(0.0,0.7)$) -- ($(bar)+(state)$);
\coordinate (action) at (0.6, 0.1);
\draw[->, line width = 0.8pt, opacity=.8, OliveGreen] ($(bar)+(action)+(0.0,0.7)$) -- ($(bar)+(action)$);
\draw[-, line width = 0.8pt, opacity=.8] ($(bar)+(0.0,0.3)$) -- ($(bar)+(0.0,0.1)$);
\draw[-, dashed, line width = 0.8pt, opacity=.65] ($(bar)+(0.0,0.1)$) -- ($(bar)+(0.0,-1.3)$);

\node[] at ($(bar)+(0.0,0.5)$) {\footnotesize $0$};

\draw[->, color = imblue, line width = 1.5] ($(bar)+(0.0,-0.3)$) -- ($(state)+(bar)+(0.0,-0.4)$) node[midway, below]{\footnotesize \textcolor{imblue}{$x_{t+1} = x_{t}$}};
\draw[->, color = OliveGreen, line width = 1.5] ($(bar)+(0.0,-0.7)$) -- ($(action)+(bar)+(0.0,-0.8)$) node[midway, below]{\footnotesize \textcolor{OliveGreen}{$a_t$}};
\draw[->, color = black, line width = 1.5] ($(bar)+(0.0,-0.3)$) -- ($(bar)+(alpha)+(0.0,-0.4)$) node[midway, below]{\footnotesize \textcolor{black}{$p_{t+1} = p_{t}$}};
\draw[->, color = imblue, line width = 1.5] ($(bar)+(alpha)+(0.0,-1.2)$) -- ($(state)+(bar)+(0.0,-1.2)$) node[midway, below]{\footnotesize \textcolor{imblue}{$\alpha$}};

\draw[black, line width=0.5pt] (bar_l) rectangle (bar_r); 

\end{scope}

\end{tikzpicture}
\caption{Sketch of a time-step interaction between an agent and two different tasks $\mathcal{M}$ (\textbf{A} and \textbf{B}) sampled in $\mathcal{D}$ for the first benchmark. Each task is defined by the bias $\alpha$ on the target's position $p_t$ observed by the agent. $x_t$ is the observation made by the agent at time-step $t$ and $a_t$ its action. For these examples, $a_t$ falls outside the target area (the zone delimited by the dashed lines), and thus the reward $r_t$ received by the agent is equal to $-|a_t-p_t|$ and $p_{t+1} = p_{t}$. If the agent had taken an action near the target, then it would have received a reward equal to $10$ and the position of the target would have been re-sampled uniformly in $[-5-\alpha, 5-\alpha]$.} 
\label{fig:pres_bench1}
\end{figure}
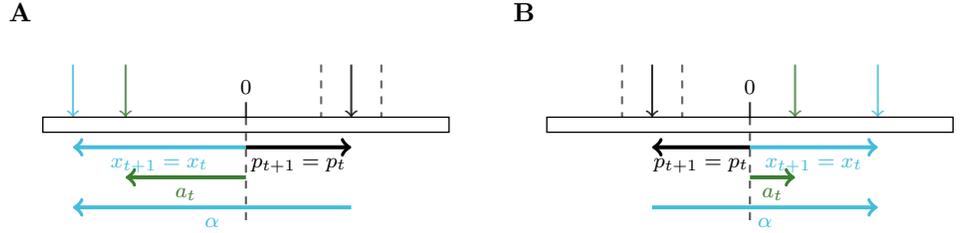

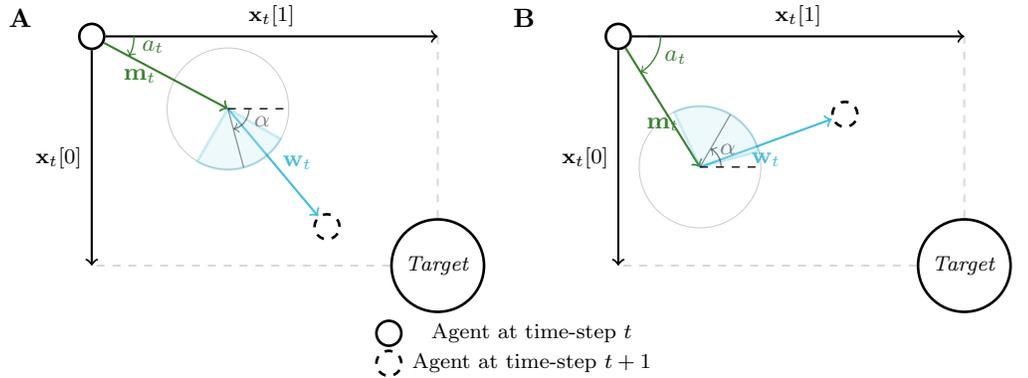
\begin{figure}[htbp]
\centering
\definecolor{imblue}{HTML}{44BDDA}
\begin{tikzpicture}[scale=1, minimum size=1em]
\draw[white, line width=1pt] (-6,0) rectangle (7.75,4); 
\begin{scope}
\node[] at (-5.7,5.0) {\textbf{A}};
\node[line width=1pt,minimum size=3.5em,circle, draw=black] at (-0.15,1.7) (target) {\footnotesize \textit{Target}};
\node[line width=1pt,minimum size=0.5em,circle, draw=black] at (-4.75,4.75) (agent) {};
\coordinate (link_b) at (-4.75,1.7);
\coordinate (link_r) at (-0.15,4.75);
\draw[->,line width=0.8pt] (agent) -- (link_b) node[midway, left] {\footnotesize $\mb{x}_t[0]$};
\draw[->,line width=0.8pt] (agent) -- (link_r) node[midway, above] {\footnotesize $\mb{x}_t[1]$};
\draw[dashed,line width=0.8pt, opacity=.15] (target) -- (link_b) {};
\draw[dashed,line width=0.8pt, opacity=.15] (target) -- (link_r) {};

\draw[->, line width=0.8pt, color=OliveGreen] (agent) --+ ([shift=(-28:2.05)]0,0) node[midway, left] {\small $\mb{m}_t$};

\node[line width=1pt,minimum size=0.5em,circle, draw=black, dashed] at ($(agent)+([shift=(-28:2.05)]0,0)+([shift=(-50:2.05)]0,0)$) (agentn) {};
\draw[->, line width=0.8pt, color=imblue] ($(agent)+([shift=(-28:2.05)]0,0)$) -- (agentn) node[midway, right] {\small $\mb{w}_t$};

\coordinate (center_circle) at ($(agent) + ([shift=(-28:2.05)]0,0)$);
\draw[opacity=.2] (center_circle) circle (2.3em);
\draw[imblue, line width=1pt, opacity=.3, fill=imblue!30] ([shift=(-30:2.3em)]$(center_circle)$) arc[start angle=-30, end angle=-120, radius=2.3em] -- (center_circle) -- cycle;
\draw[opacity=.4, line width=0.5pt] (center_circle) -- ([shift=(-75:2.3em)]$(center_circle)$);
\draw[dashed, line width=0.5pt] (center_circle) -- ([shift=(0:2.3em)]$(center_circle)$);

\draw[->, black, line width=0.5pt, opacity=.5] ([shift=(0:0.8em)]$(center_circle)$) arc[start angle=0, end angle=-75, radius=0.8em] node[midway, right] {\small $\alpha$};
\draw[->,black, line width=0.5pt, opacity=1., color=OliveGreen] ([shift=(0:1.6em)]$(agent)$) arc[start angle=0, end angle=-28, radius=1.6em] node[midway, right] {\small $a_t$};

\end{scope}

\begin{scope}
\node[] at (1.0,5.0) {\textbf{B}};	    

\node[line width=1pt,minimum size=3.5em,circle, draw=black] at (6.85,1.7) (target) {\footnotesize \textit{Target}};
\node[line width=1pt,minimum size=0.5em,circle, draw=black] at (2.25,4.75) (agent) {};
\coordinate (link_b) at (2.25,1.7);
\coordinate (link_r) at (6.85,4.75);
\draw[->,line width=0.8pt] (agent) -- (link_b) node[midway, left] {\footnotesize $\mb{x}_t[0]$};
\draw[->,line width=0.8pt] (agent) -- (link_r) node[midway, above] {\footnotesize $\mb{x}_t[1]$};
\draw[dashed,line width=0.8pt, opacity=.15] (target) -- (link_b) {};
\draw[dashed,line width=0.8pt, opacity=.15] (target) -- (link_r) {};

\draw[->, line width=0.8pt, color=OliveGreen] (agent) --+ ([shift=(-58:2.05)]0,0) node[midway, below] {\small $\mb{m}_t$};
\node[line width=1pt,minimum size=0.5em,circle, draw=black, dashed] at ($(agent)+([shift=(-58:2.05)]0,0)+([shift=(20:2.05)]0,0)$) (agentn) {};

\draw[->, line width=0.8pt, color=imblue] ($(agent)+([shift=(-58:2.05)]0,0)$) -- (agentn) node[midway, below] {\small $\mb{w}_t$};

\coordinate (center_circle) at ($(agent) + ([shift=(-58:2.05)]0,0)$);
\draw[opacity=.2] (center_circle) circle (2.3em);
\draw[imblue, line width=1pt, opacity=.3, fill=imblue!30] ([shift=(15:2.3em)]$(center_circle)$) arc[start angle=15, end angle=115, radius=2.3em] -- (center_circle) -- cycle;
\draw[opacity=.4, line width=0.5pt] (center_circle) -- ([shift=(60:2.3em)]$(center_circle)$);
\draw[dashed, line width=0.5pt] (center_circle) -- ([shift=(0:2.3em)]$(center_circle)$);

\draw[->, black, line width=0.5pt, opacity=.5] ([shift=(0:0.8em)]$(center_circle)$) arc[start angle=0, end angle=60, radius=0.8em] node[right] {\small $\alpha$};
\draw[->,black, line width=0.5pt, opacity=1., color=OliveGreen] ([shift=(0:1.6em)]$(agent)$) arc[start angle=0, end angle=-58, radius=1.6em] node[midway, right] {\small $a_t$};

\draw[dashed, line width = 0.5, opacity=.5] (agent) --+ (0.5,0);

\end{scope}

\node[line width=1pt,minimum size=0.5em,circle, draw=black] at (-0.8,0.8) (agent) {};
\node[] at (1.1,0.8) (agent) {\footnotesize{Agent at time-step $t$}};
\node[dashed,line width=1pt,minimum size=0.5em,circle, draw=black] at (-0.8,0.4) (agent) {};
\node[] at (1.1,0.4) (agent) {\footnotesize{Agent at time-step $t+1$}};

\end{tikzpicture}
\caption{Sketch of a time-step interaction between an agent and two different tasks $\mathcal{M}$ (\textbf{A} and \textbf{B}) sampled in $\mathcal{D}$ for the second benchmark. Each task is defined by the main direction $\alpha$ of a wind cone from which a perturbation vector $\mb{w}_t$ is sampled at each time-step. This perturbation vector is then applied to the movement $m_t$ of the agent, whose direction is given by the action $a_t$. If the agent reaches the target, it receives a reward of $100$, otherwise a reward of $-2$.}
\label{fig:pres_bench2}
\end{figure}

\begin{figure}[htbp]
\centering
\definecolor{imblue}{HTML}{44BDDA}
\begin{tikzpicture}[scale=1, minimum size=1em]
\draw[white, line width=1pt] (-6,0) rectangle (7.75,4); 
\begin{scope}
\node[] at (-5.7,5.0) {\textbf{A}};


\coordinate (target1) at (-4.35,3.35);
\coordinate (target2) at (-2.3,2.0);
\coordinate (agent) at (-0.15,4.9);

\node[line width=1pt, minimum size=3.0em, circle, draw=imblue] at (target1) (target1n) {\footnotesize \textit{\textcolor{imblue}{Target 1}}};
\node[line width=1pt, minimum size=3.0em,  circle, draw=red] at (target2) (target2n){\footnotesize \textit{\textcolor{red}{Target 2}}};

\node[line width=1pt, draw=black, circle] at (agent) (agentn){};

\draw[->, line width = 0.8, color=imblue] let \p1=(agentn), \p2=(target1) in (agentn) -- (\x1,\y2) node[midway, right] {\small \footnotesize \textcolor{imblue}{$\mb{x}_t[0]$}};
\draw[->, line width = 0.8, color=imblue] let \p1=(agentn), \p2=(target1) in (agentn) -- (\x2,\y1) node[midway, above] {\small \footnotesize \textcolor{imblue}{$\mb{x}_t[1]$}};
\draw[->, line width = 0.8, color=red] let \p1=(agentn), \p2=(target2) in (agentn) -- (\x1,\y2) node[midway, right] {\small \footnotesize \textcolor{red}{$\mb{x}_t[2]$}};
\draw[->, line width = 0.8, color=red] let \p1=(agentn), \p2=(target2) in (agentn) -- (\x2,\y1) node[midway, above] {\small \footnotesize \textcolor{red}{$\mb{x}_{t}[3]$}};
\draw[->, line width = 0.8, color=imblue] let \p1=(agentn), \p2=(target1) in (agentn) -- (\x1,\y2) node[midway, right] {\small \footnotesize \textcolor{imblue}{$\mb{x}_t[0]$}};

\draw[-,dashed, line width = 0.8, opacity=.15] let \p1=(agentn), \p2=(target1) in (target1n) -- (\x1,\y2);
\draw[-,dashed, line width = 0.8,opacity=.15] let \p1=(agentn), \p2=(target1) in (target1n) -- (\x2,\y1);
\draw[-,dashed, line width = 0.8,opacity=.15] let \p1=(agentn), \p2=(target2) in (target2n) -- (\x1,\y2);
\draw[-, dashed, line width = 0.8,opacity=.15] let \p1=(agentn), \p2=(target2) in (target2n) -- (\x2,\y1);

\draw[->,black, line width=0.8pt, opacity=1., color=OliveGreen] ([shift=(0:1.6em)]$(agent)$) arc[start angle=0, end angle=-143, radius=1.6em] node[midway, right] {\footnotesize $a_t$};

\coordinate (new_agent) at ($(agent) + ([shift=(-143:1.5)]0,0)$);
\node[dashed,line width=1pt, draw=black, circle] at (new_agent) (aa){};

\draw[->, color=OliveGreen, line width = 0.8, opacity = 1.] (agentn) -- (aa);

\draw[->,black, line width=0.8pt, opacity=1., color=OliveGreen] ([shift=(0:1.6em)]$(new_agent)$) arc[start angle=0, end angle=-180, radius=1.6em] node[midway, below] {\footnotesize $a_{t+1}$};

\draw[->, color=OliveGreen, line width = 0.8, opacity = 1.] (agentn) -- (aa);
\coordinate (new_aa) at ($(new_agent) + ([shift=(-180:1.5)]0,0)$);
\draw[->, color=OliveGreen, line width = 0.8, opacity = 1.] (aa) -- (new_aa);

\draw[-, dashed, line width = 0.5, opacity = .3] (agentn) -- ($(agentn) + (0.5,0.0)$);
\draw[-, dashed, line width = 0.5, opacity = .3] (aa) -- ($(aa) + (0.5,0.0)$);

\end{scope}

\begin{scope}
\node[] at (1.0,5.0) {\textbf{B}};	    


\coordinate (target1) at (4.5,4.55);
\coordinate (target2) at (6.55,2.75);
\coordinate (agent) at (2.35,1.5);

\node[line width=1pt, minimum size=3.0em, circle, draw=red] at (target1) (target1n) {\footnotesize \textit{\textcolor{red}{Target 1}}};
\node[line width=1pt, minimum size=3.0em,  circle, draw=imblue] at (target2) (target2n){\footnotesize \textit{\textcolor{imblue}{Target 2}}};

\node[line width=1pt, draw=black, circle] at (agent) (agentn){};
\draw[->, line width = 0.8, color=red] let \p1=(agentn), \p2=(target1) in (agentn) -- (\x1,\y2) node[midway, left] {\small \footnotesize \textcolor{red}{$\mb{x}_t[0]$}};
\draw[->, line width = 0.8, color=imblue] let \p1=(agentn), \p2=(target2) in (agentn) -- (\x1,\y2) node[midway, left] {\small \footnotesize \textcolor{imblue}{$\mb{x}_t[2]$}};
\draw[->, line width = 0.8, color=imblue] let \p1=(agentn), \p2=(target2) in (agentn) -- (\x2,\y1) node[midway, below] {\small \footnotesize \textcolor{imblue}{$\mb{x}_{t}[3]$}};
\draw[->, line width = 0.8, color=red] let \p1=(agentn), \p2=(target1) in (agentn) -- (\x2,\y1) node[midway, below] {\small \footnotesize \textcolor{red}{$\mb{x}_t[1]$}};

\draw[-,dashed, line width = 0.8, opacity=.15] let \p1=(agentn), \p2=(target1) in (target1n) -- (\x1,\y2);
\draw[-,dashed, line width = 0.8,opacity=.15] let \p1=(agentn), \p2=(target1) in (target1n) -- (\x2,\y1);
\draw[-,dashed, line width = 0.8,opacity=.15] let \p1=(agentn), \p2=(target2) in (target2n) -- (\x1,\y2);
\draw[-, dashed, line width = 0.8,opacity=.15] let \p1=(agentn), \p2=(target2) in (target2n) -- (\x2,\y1);

\draw[->,black, line width=0.8pt, opacity=1., color=OliveGreen] ([shift=(0:1.6em)]$(agent)$) arc[start angle=0, end angle=28, radius=1.6em] node[midway, right] {\footnotesize $a_t$};

\coordinate (new_agent) at ($(agent) + ([shift=(28:1.5)]0,0)$);
\node[dashed,line width=1pt, draw=black, circle] at (new_agent) (aa){};

\draw[->,black, line width=0.8pt, opacity=1., color=OliveGreen] ([shift=(0:1.6em)]$(new_agent)$) arc[start angle=0, end angle=60, radius=1.6em] node[midway, right] {\footnotesize $a_{t+1}$};

\draw[->, color=OliveGreen, line width = 0.8, opacity = 1.] (agentn) -- (aa);
\coordinate (new_aa) at ($(new_agent) + ([shift=(60:1.5)]0,0)$);
\draw[->, color=OliveGreen, line width = 0.8, opacity = 1.] (aa) -- (new_aa);

\draw[-, dashed, line width = 0.5, opacity = .3] (agentn) -- ($(agentn) + (0.5,0.0)$);
\draw[-, dashed, line width = 0.5, opacity = .3] (aa) -- ($(aa) + (0.5,0.0)$);
\end{scope}

\coordinate (legend_b) at (0.0, 1.0);
\coordinate (legend_r) at ($(legend_b)+(0,-0.4)$);
\draw[fill = imblue] (legend_b) circle(0.2em);
\draw[fill = red] (legend_r) circle(0.2em);
\node[] at ($(legend_b)+(0.7,0.0)$) {\footnotesize $r = 100$};
\node[] at ($(legend_r)+(0.73,0.0)$) {\footnotesize $r = -50$};
\end{tikzpicture}
\caption{Sketch of a time-step interaction between an agent for the two different tasks $\mathcal{M}$ (\textbf{A} and \textbf{B}) sampled in $\mathcal{D}$ for the third benchmark. Each task is defined by the attribution of a positive reward to one of the two targets (in blue) and a negative reward to the other (in red). At each time-step the agent outputs an action $a_t$ which drives the direction of its next move. If the agent reaches a target, it receives the corresponding reward.}
\label{fig:pres_bench3}
\end{figure}
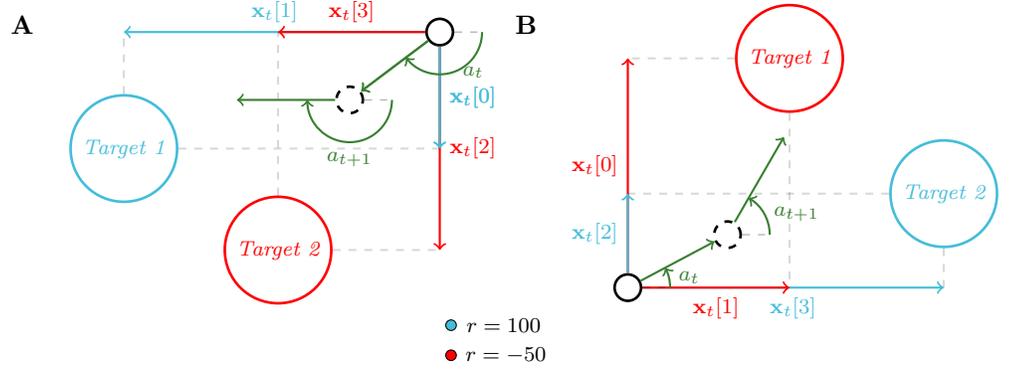

\paragraph{Benchmark 1.}
We define the first benchmark (made of a 1-D state space and action space) through a random variable $\alpha$, informative enough to distinguish all different MDPs in $\mathcal{D}$. With this definition, $\alpha$ represents the current task and drawing $\alpha$ at the beginning of each episode amounts to sampling a new task in $\mathcal{D}$. At each time-step, the agent observes a biased version $x_t = p_t + \alpha$ of the exact position of a target $p_t$ belonging to the interval $[-5 - \alpha, 5 - \alpha]$, with $\alpha \sim \mathbb{U}[-10,10]$. The agent outputs an action $a_t \in \left[-20, 20\right]$ and receives a reward $r_t$ which is equal to $10$ if $|a_t - p_t| < 1$ and $-|a_t - p_t|$ otherwise. In case of positive reward, $p_{t + 1}$ is re-sampled uniformly in its domain, else $p_{t+1} = p_t$. This benchmark is represented on Figure~\ref{fig:pres_bench1}.

\paragraph{Benchmark 2.}
The second benchmark consists of navigating towards a target in a 2-D space with noisy movements. Similarly to the first benchmark, we can distinguish all different MDPs in $\mathcal{D}$ through a three-dimensional random vector of variables $\boldsymbol{\alpha}$. The target is placed at $(\boldsymbol{\alpha}[1],\boldsymbol{\alpha}[2])$ in the 2-D space. At each time-step, the agent observes its relative position to the target and outputs the direction of a move vector $\mb{m}_t$. A perturbation vector $\mb{w}_t$ is then sampled uniformly in a cone, whose main direction $\boldsymbol{\alpha}[3] \sim \mathbb{U}[-\pi,\pi[$, together with the target's position, define the current task in $\mathcal{D}$. Finally the agent is moved following $\mb{m}_t + \mb{w}_t$ and receives a reward ($r_t = -0.2$). If the agent reaches the target, it instead receives a high reward ($r_t = 100$) and is moved to a position sampled uniformly in the 2-D space. This benchmark is represented on Fig~\ref{fig:pres_bench2}

\paragraph{Benchmark 3.}
The third benchmark also involves navigating in a 2-D space, but which contains two targets. As for the two previous benchmarks, we distinguish all different MDPs in $\mathcal{D}$ through a five-dimensional random vector of variables $\boldsymbol{\alpha}$. The targets are placed at positions $(\boldsymbol{\alpha}[1],\boldsymbol{\alpha}[2])$ and $(\boldsymbol{\alpha}[3],\boldsymbol{\alpha}[4])$. At each time-step, the agent observes its relative position to the two targets and is moved along a direction given by its action. One target, defined by the task in $\mathcal{D}$ through $\boldsymbol{\alpha}[5]$, is attributed a positive reward ($100$) and the other a negative reward ($-50$). In other words, $\boldsymbol{\alpha}[5]$ is a Bernoulli variable that determines which target is attributed the positive reward and which is attributed the negative one. As for benchmark 2, once the agent reaches a target, it receives the corresponding reward and is moved to a position sampled uniformly in the 2-D space. This benchmark is represented on Figure~\ref{fig:pres_bench3}.

\section{Results}
\label{sec:results}
\paragraph{Learning.}
\begin{figure}
\centering
\includegraphics[width = 1.0\textwidth]{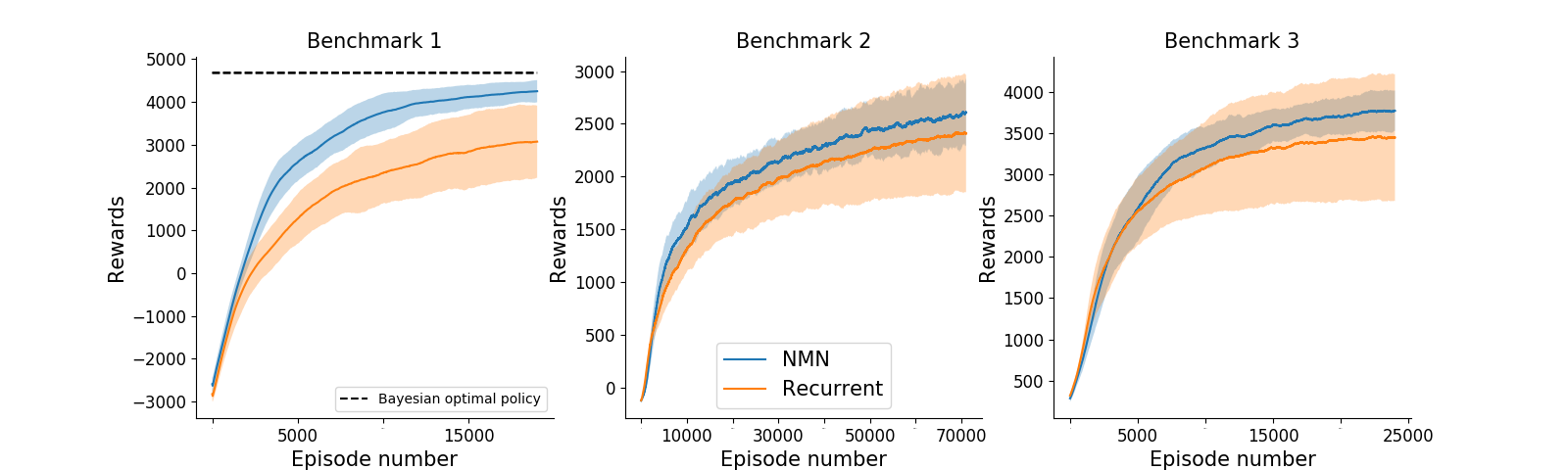}
\caption{Mean ($\pm$ std in shaded) sum of rewards obtained over $15$ training runs with different random seeds with respect to the episode number. Results of benchmark 1,2 and 3 are displayed from left to right. The plots are smoothed thanks to a running mean over $1000$ episodes.}
\label{fig:results}
\end{figure}

From a learning perspective, a comparison of the sum of rewards obtained per episode by NMNs and RNNs on the three benchmarks is shown in Figure~\ref{fig:results}. Results show that, on average, NMNs learn faster (with respect to the number of episodes) and converge towards better policies than RNNs (i.e., higher rewards for the last episodes). It is worth mentioning that, NMNs show very stable results, with small variances over different random seeds, as opposed to RNNs. To put the performance of the NMN in perspective, we note that an optimal Bayesian policy would achieve an expected sum of rewards of $4679$ on benchmark 1 (see \ref{app:opt} for proof) whereas NMNs reach, after $20000$ episodes, an expected sum of rewards of $4534$. For this simple benchmark, NMNs manage to learn near-optimal Bayesian policies.

\paragraph{Adaptation.}
From an adaptation perspective, Figure~\ref{fig:results_nmd} shows the temporal evolution of the neuromodulatory signal $\mathbf{z}$ (part \textbf{A}), of the scale factor (for each neuron of a hidden layer, part \textbf{B}) and of the rewards (part \textbf{C}) obtained with respect to $\alpha$ for $1000$ episodes played on benchmark $1$. 
For small values of $t$, the agent has little information on the current task, leading to a non-optimal behaviour (as it can be seen from the low rewards). 
Of greatest interest, the signal $\mb{z}$ for the first time-steps exhibits little dependence on $\alpha$, highlighting the agent uncertainty on the current task and translating to noisy scale factors. 
Said otherwise, for small $t$, the agent learned to play a (nearly) task-independent strategy. As time passes, the agent gathers further information about the current task and approaches a near-optimal policy. This is reflected in the convergence of $\mb{z}$ (and thus scale factors) with a clear dependency on $\alpha$ and also in wider-spread values of $\mb{z}$. 
For a large value of $t$, $\mb{z}$ holding constant between time-steps shows that the neuromodulatory signal is almost state-independent and serves only for adaptation. We note that the value of $\mathbf{z}$ in each of its dimensions varies continuously with $\alpha$, meaning that for two similar tasks, the signal will converge towards similar values. Finally, it is interesting to look at the neurons scale factor variation with respect to $\alpha$ (\textbf{B}). Indeed, for some neurons, one can see that the scale factors vary between negative and positive values, effectively inverting the slope of the activation function. Furthermore, it is interesting to see that some neurons are inactive (scale factor almost equal to $0$, leading to a constant activation function) for some values of $\alpha$.

\begin{figure}[H]
\centering
\includegraphics[width = \textwidth]{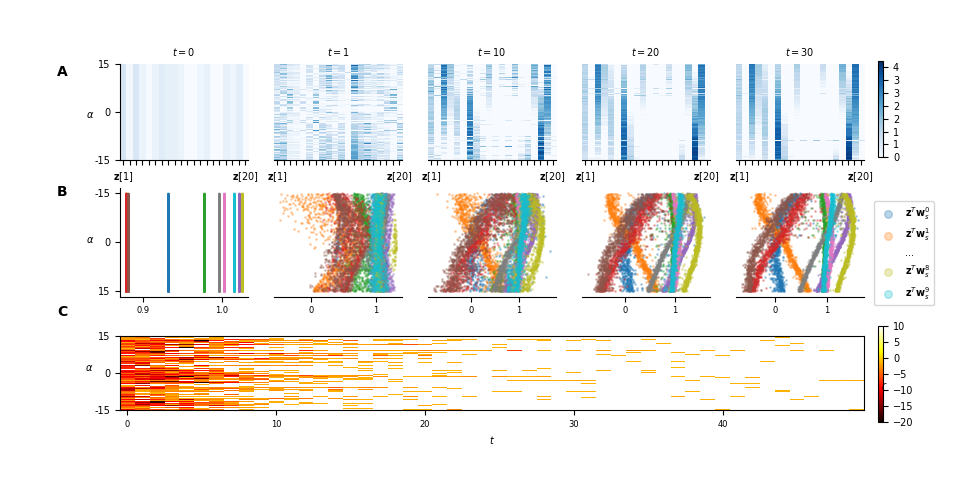}
\caption{Adaptation capabilities of the NMN architecture on benchmark 1. \textbf{A.} Temporal evolution of the neuromodulatory signal $\mathbf{z}$ with respect to $\alpha$, gathered on $1000$ different episodes. Note that the neuromodulatory signals go from uniform distributions over all possible $\alpha$ values (i.e., the different contexts) to non-uniform and adapted (w.r.t. $\alpha$) distributions along with an increase in the rewards. \textbf{B.} The value of the scale factors with respect to $\alpha$ for each neuron of a hidden layer in the main network.
\textbf{C.} Rewards obtained at each time-step by the agent during those episodes. Note that light colours represent high rewards and correspond to adapated neuromodulatory signals.} 
\label{fig:results_nmd}
\end{figure}

For benchmark 2, let us first note that $\mb{z}$ seems to code exclusively for $\boldsymbol{\alpha}[3]$. Indeed, $\mb{z}$ converges slowly with time with respect to $\boldsymbol{\alpha}[3]$, whatever the value of $\boldsymbol{\alpha}[1]$ and $\boldsymbol{\alpha}[2]$ (Figure~\ref{fig:results_nmd_b2}). This, could potentially be explained by the fact that one does not need the values of $\boldsymbol{\alpha}[1]$ and $\boldsymbol{\alpha}[2]$ to compute an optimal move. The graphs on Figure~\ref{fig:results_nmd_b2} are projected on the dimension $\boldsymbol{\alpha}[3]$, allowing the same analysis as for benchmark 1.

The results obtained for benchmark 2 (Figure~\ref{fig:results_nmd_b2}) show similar characteristics. Indeed, despite the agent receiving only noisy information on $\boldsymbol{\alpha}[3]$ at each time-step (as perturbation vectors are sampled uniformly in a cone centered on $\boldsymbol{\alpha}[3]$), $\mb{z}$ quasi-converges slowly with time (part \textbf{A}). The value of $\mathbf{z}$ in each of its dimensions also varies continuously with $\boldsymbol{\alpha}[3]$ (as for the first benchmark) resulting also in continuous scale factors variations. This is clearly highlighted at time-step $100$ on Figure~\ref{fig:results_nmd_b2} where the scale factors of some neurons appear highly asymmetric, but with smooth variations with respect to $\boldsymbol{\alpha}[3]$. Finally, let us highlight that for this benchmark, the agent continues to adapt even when it is already performing well. Indeed, one can see that after $40$ time-steps the agent is already achieving good results (part \textbf{C}), even though $\mb{z}$ has not yet converged (part \textbf{A}), which is due to the stochasticity of the environment. Indeed, the agent only receives noised information on $\alpha$ and thus after $40$ time-steps it has gathered sufficient information to act well on the environment, but insufficient information to deduce a near-exact value of $\boldsymbol{\alpha}[3]$. This shows that the agent can perform well, even while it is still gathering relevant information on the current task.

\begin{figure}[H]
\centering
\includegraphics[width = 1.01\textwidth]{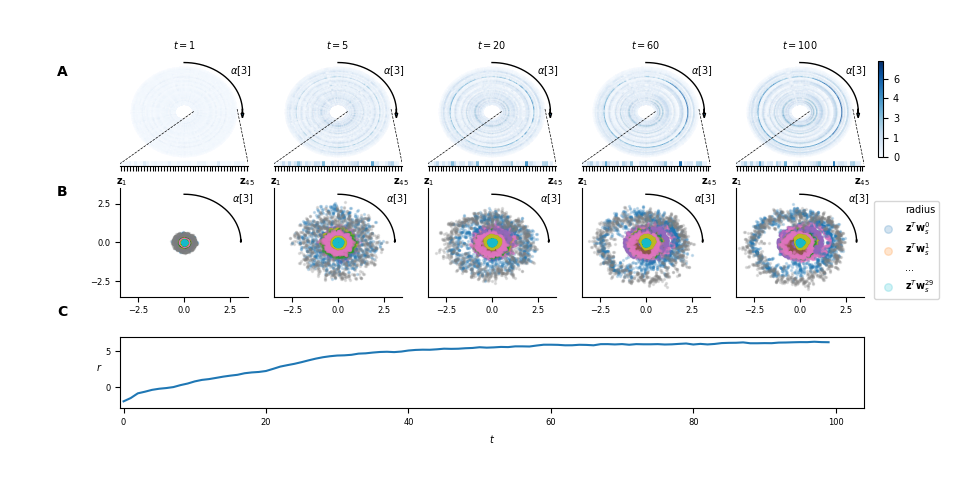}
\caption{Adaptation capabilities of the NMN architecture on benchmark 2. \textbf{A.} Temporal evolution of the neuromodulatory signal $\mathbf{z}$ with respect to $\boldsymbol{\alpha}[3]$, gathered on $1000$ different episodes. As $\boldsymbol{\alpha}[3]$ is an angle, the plot is projected in polar coordinates for a better interpretability of the results. Each dimension of $\mb{z}$ is corresponds to a different radius. \textbf{B.} The value of the scale factors with respect to $\boldsymbol{\alpha}[3]$ for each neuron of a hidden layer in the main network. Again, the plot is projected in polar coordinates. For a given $\boldsymbol{\alpha}[3]$, the values of the neurons' scale factor are given thanks to the radius.
\textbf{c.} Average reward obtained at each time-step by the agent during those episodes. Note that after an average of $40$ time-steps, the agent is already achieving decent performances even though $\mb{z}$ has not yet converged.} 
\label{fig:results_nmd_b2}
\end{figure}

It is harder to interpret the neuromodulatory signal for benchmark 3. In fact, for that benchmark, we show that the signal seems to code not only for the task in $\mathcal{D}$ but also for the state of the agent in some sense. As $\boldsymbol{\alpha}$ is five-dimensional, it would be very difficult to look at its impact on $\mb{z}$ as a whole. Rather, we fix the position of the two references in the 2-D space and look at the behaviour of $\mb{z}$ with respect to $\boldsymbol{\alpha}[5]$. In Figure~\ref{fig:results_nmd_3} adaptation is clearly visible in the rewards obtained by the agent (part \textbf{C}) with very few negative rewards after $30$ time-steps. We note that for later time-steps, $\mb{z}$ tends to partially converge (\textbf{A}) and :
\begin{itemize}
    \item some dimensions of $\mb{z}$ are constant with respect to $\boldsymbol{\alpha}[5]$, indicating that they might be coding for features related to $\boldsymbol{\alpha}[1,2,3,4]$.
    \item Some other dimensions are well correlated to $\boldsymbol{\alpha}[5]$, for which similar observations than for the two other benchmarks can be made. For example, one can see that some neurons have a very different scale factors for the two possible different values of $\boldsymbol{\alpha}[5]$ (\textbf{B}).
    \item The remaining dimensions do not converge at all, implying that these are not related to $\boldsymbol{\alpha}$, but rather to the state of the agent.
\end{itemize}

\begin{figure}[H]
\centering
\includegraphics[width = 1.01\textwidth]{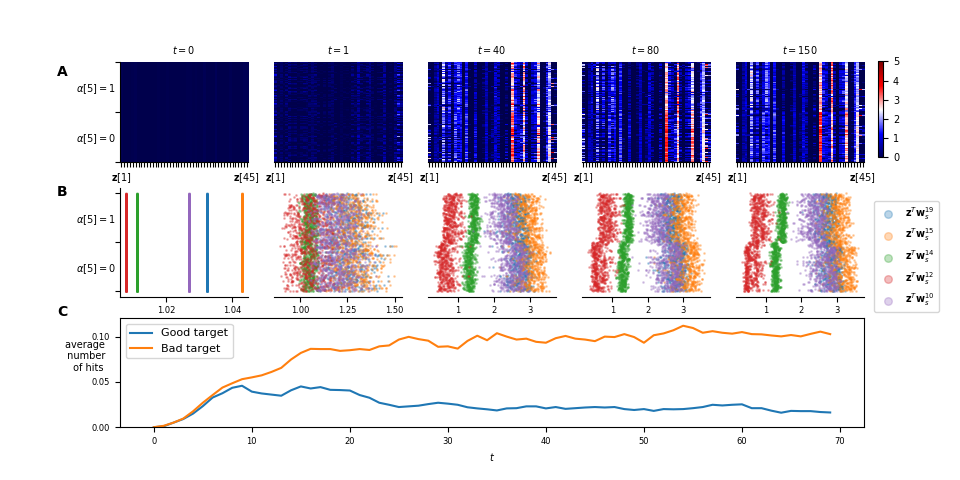}
\caption{Adaptation capabilities of the NMN architecture on benchmark 3. \textbf{A.} Temporal evolution of the neuromodulatory signal $\mathbf{z}$ with respect to $\boldsymbol{\alpha}[5]$, gathered on $1000$ different episodes. Note that the neuromodulatory signals go from uniform distributions over all possible alpha values (i.e., the different contexts) to non-uniform and adapted (w.r.t. alpha) distributions along with an increase of the rewards. \textbf{B.} The value of the scale factors with respect to $\boldsymbol{\alpha}[5]$ for the $5$ neurons of a hidden layer in the main network, for which the scale factor is the most correlated to $\boldsymbol{\alpha}[5]$.
\textbf{C.} Average number of good and bad target hits at each time-step during those episodes. On average, after $15$ time-steps, the agent starts navigating towards the correct target while avoiding the wrong one.} 
\label{fig:results_nmd_3}
\end{figure}

These results suggest that in this case, the neuromodulation network is used to code more complex information than simply that required to differentiate tasks, making $\mb{z}$ harder to interpret. Despite $\mb{z}$ not converging on some of its dimensions, we stress that freezing $\mb{z}$ after adaptation will not strongly degrade the agent's performance. That is, the features coded in $\mb{z}$ that do not depend on $\boldsymbol{\alpha}$ are not critical to the performance of the agent. To illustrate this, we will analyse the behaviour of the agent within an episode when freezing and unfreezing the neuromodulation signal and when changing task. This behaviour is shown on Figure~\ref{fig:freeze_nmd}, for which:

\begin{enumerate}
    \item[\textbf{(a)}] Shows the behaviour of the agent when $\mb{z}$ is locked to its initial value. This plot thus shows the initial "exploration" strategy used by the agent; that is, the strategy played by the agent when it has not gathered any information on the current task.
    \item[\textbf{(b)}] Shows the behaviour of the agent after unlocking $\mb{z}$, that is when the agent is able to adapt freely to the current task by updating $\mb{z}$ at each time-step.
    \item[\textbf{(c)}] Shows the behaviour of the agent when locking $\mb{z}$ at a random time-step after adaptation. $\mb{z}$ is thus fixed at a value which fits well the current task. As one can see, the agent continues to navigate towards the correct target. The performance is however a slightly degraded as the agent seems to lose some capacity to avoid the wrong target. This further suggests that, in this benchmark (as opposed to the two others), the neuromodulation signal does not only code for the current task but also for the current state, in some sense, that is hard to interpret.
    \item[\textbf{(d)}] Shows the same behaviour as in \textbf{(c)} as $\mb{z}$ is still locked to the same value, but the references are now switched. As there is no adaptation without updating $\mb{z}$; the agent is now always moving towards to wrong target.
    \item[\textbf{(e)}] Shows the behaviour of the agent when unlocking $\mb{z}$ once again. As one can see, the agent is now able to adapt correctly by updating $\mb{z}$ at each time-step, and thus it navigates towards the correct target once again.
\end{enumerate}

\begin{figure*}[t!]
    \centering
        \subcaptionbox{Fixed initial $\mb{z}$.}{\adjincludegraphics[height=0.83in,trim={0 0 0 {.5\height}},clip]{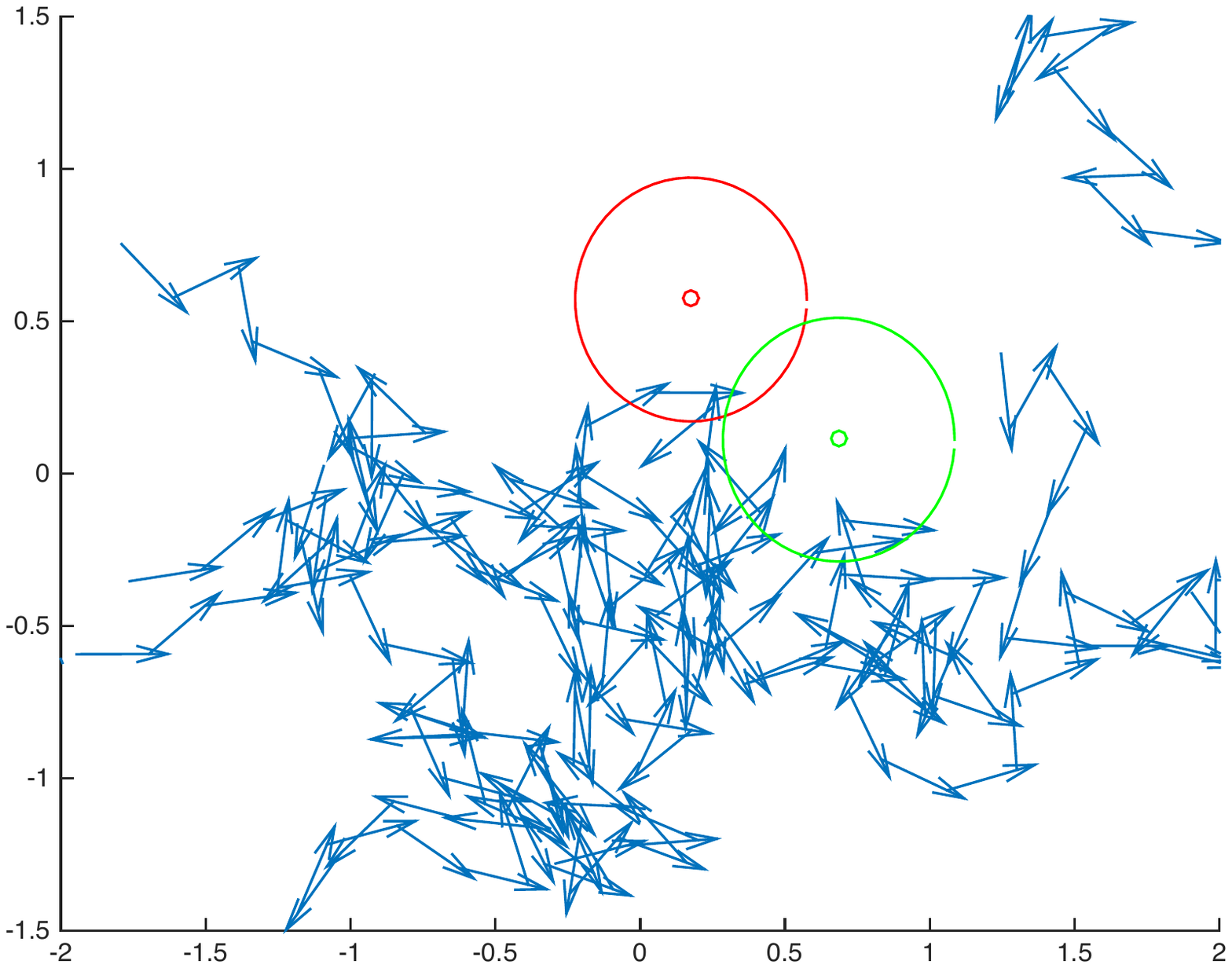}        \hspace{-7.00mm}}
        \subcaptionbox{Unlocking $\mb{z}$.}{\adjincludegraphics[height=0.83in,trim={0 0 0 {.5\height}},clip]{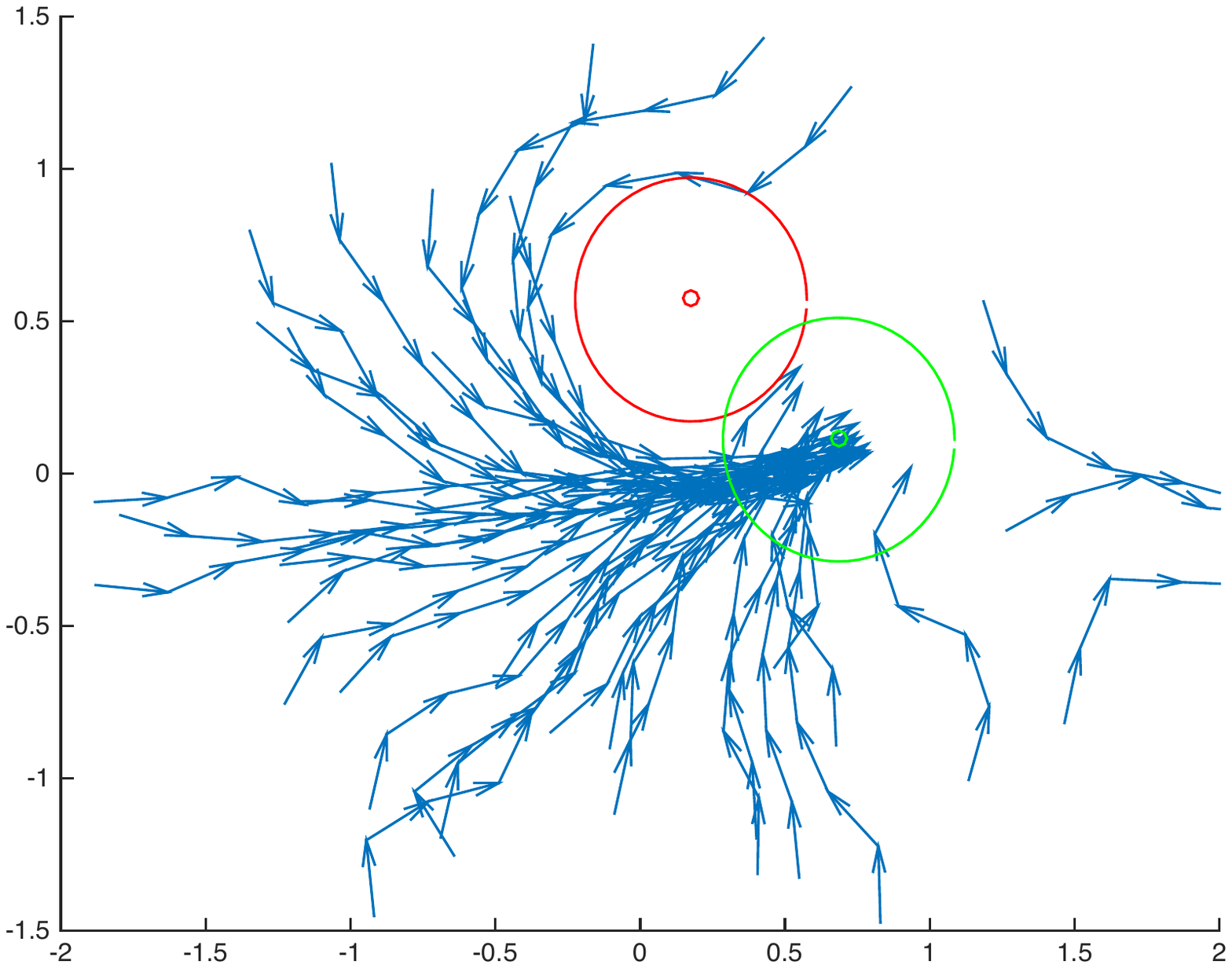}
                \hspace{-7.00mm}}
        \subcaptionbox{Locking $\mb{z}$.}{\adjincludegraphics[height=0.83in,trim={0 0 0 {.5\height}},clip]{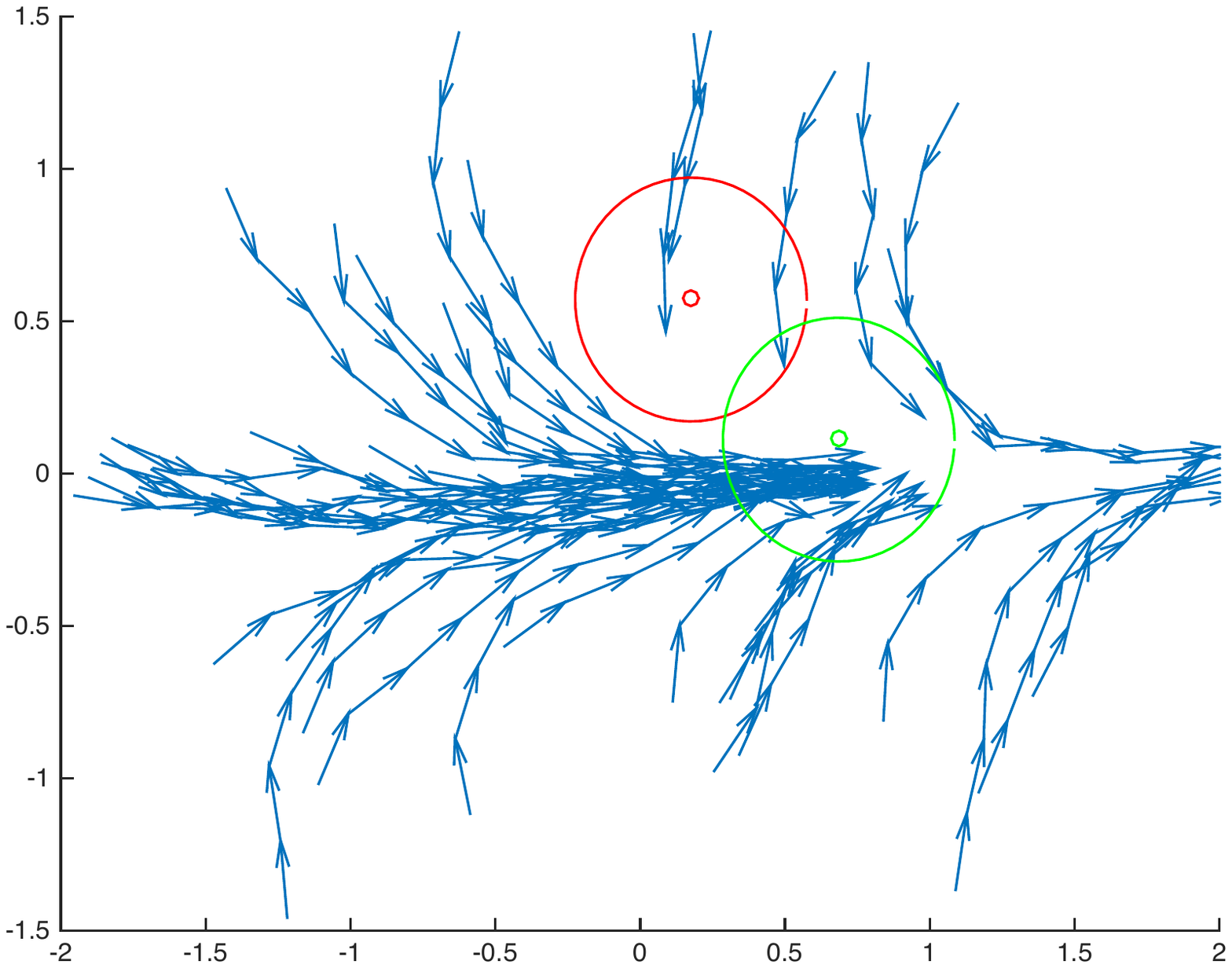}        \hspace{-7.00mm}}
        \subcaptionbox{Switching references.}{\adjincludegraphics[height=0.83in,trim={0 0 0 {.5\height}},clip]{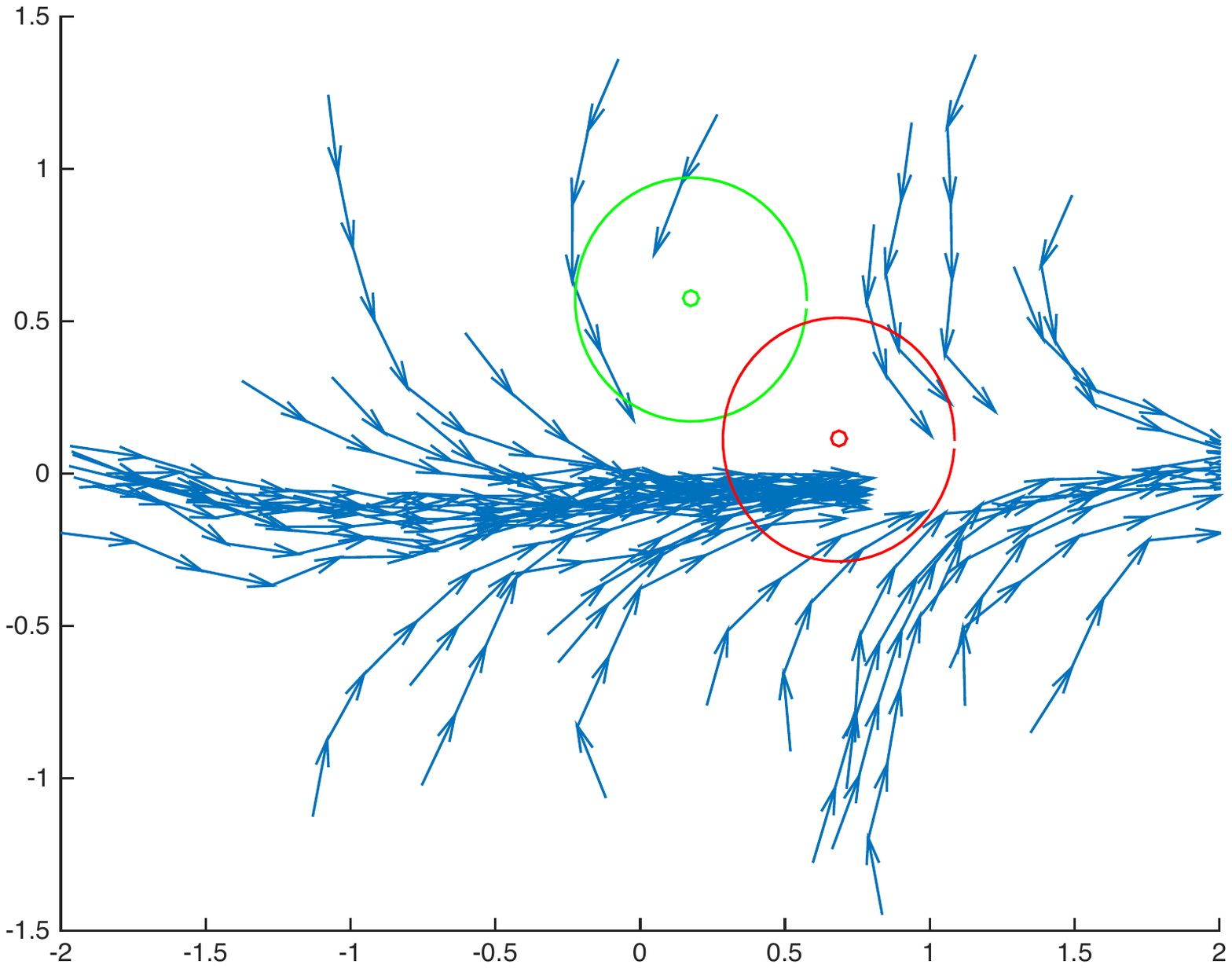}        \hspace{-7.00mm}}
        \subcaptionbox{Unlocking $\mb{z}$.}{\adjincludegraphics[height=0.83in,trim={0 0 0 {.5\height}},clip]{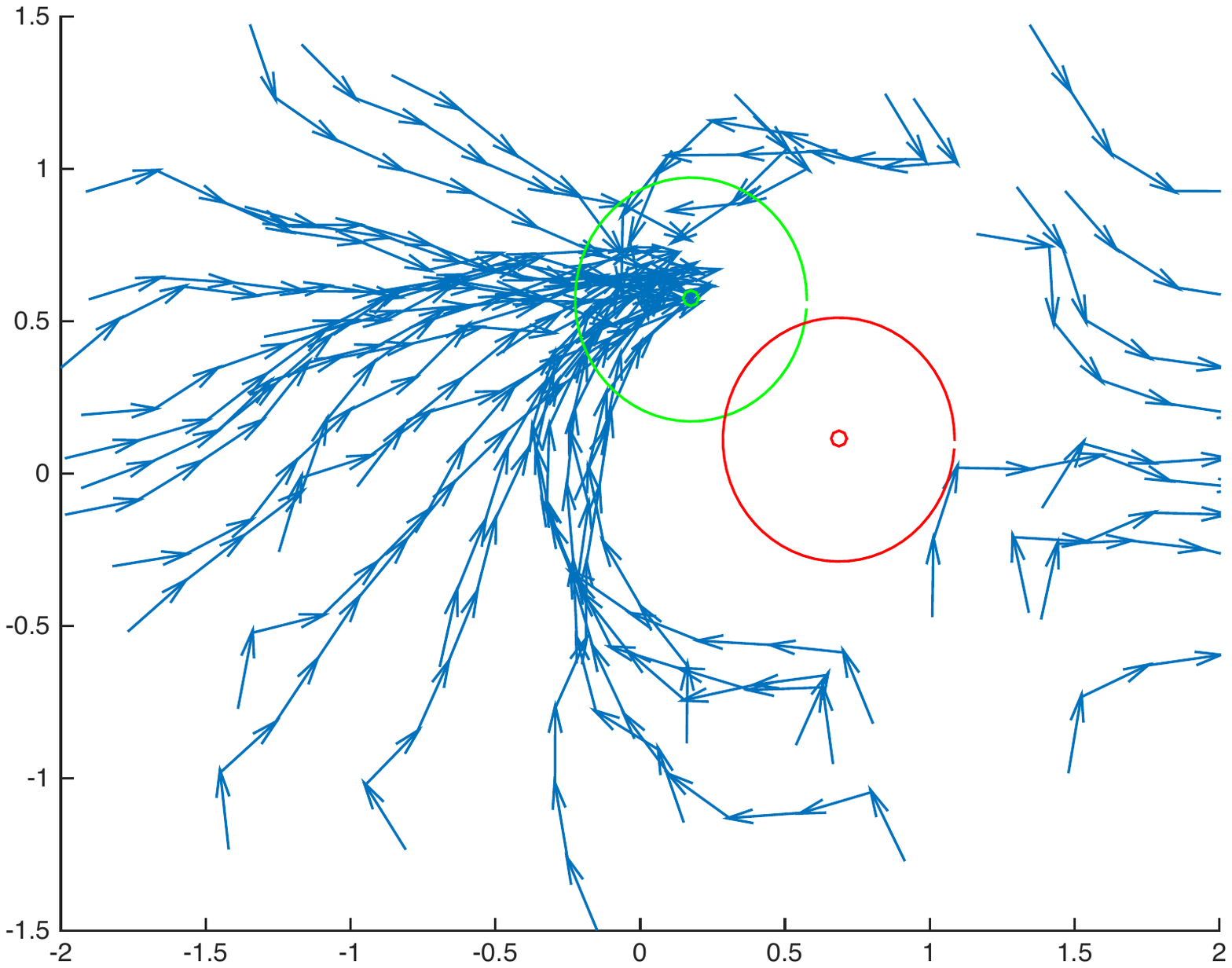}        \hspace{-7.00mm}}
    \caption{Analysis of the agent's behaviour when freezing and unfreezing the neuromodulation signal and when changing task within an episode. The green reference is attributed a reward of $100$ while the red one is attributed a reward of $-50$. Each blue arrow represents the movement of the agent for a given time-step.}
    \label{fig:freeze_nmd}
\end{figure*}

\paragraph{Robustness study.}
\label{subsec:robustness}
Even though results are quite promising for the NMN, it is interesting to see how it holds up with another type of activation function as well as analysing its robustness to different main networks' architectures.

\subparagraph{Sigmoid activation functions.}
Figure~\ref{fig:results_sigmoid} shows the comparison between having sigmoids as the main network's activation function instead of sReLUs. As one can see, sigmoid activation functions lead to worse or equivalent results to sReLUs, be it for RNNs or NMNs. In particular, the NMN architecture seems more robust to the change of activation function as opposed to RNNs, as the difference between sReLUS and sigmoids is often far inferior for NMNs than RNNs (especially for benchmark 2).

\begin{figure}[H]
\centering
\includegraphics[width = 1.0\textwidth]{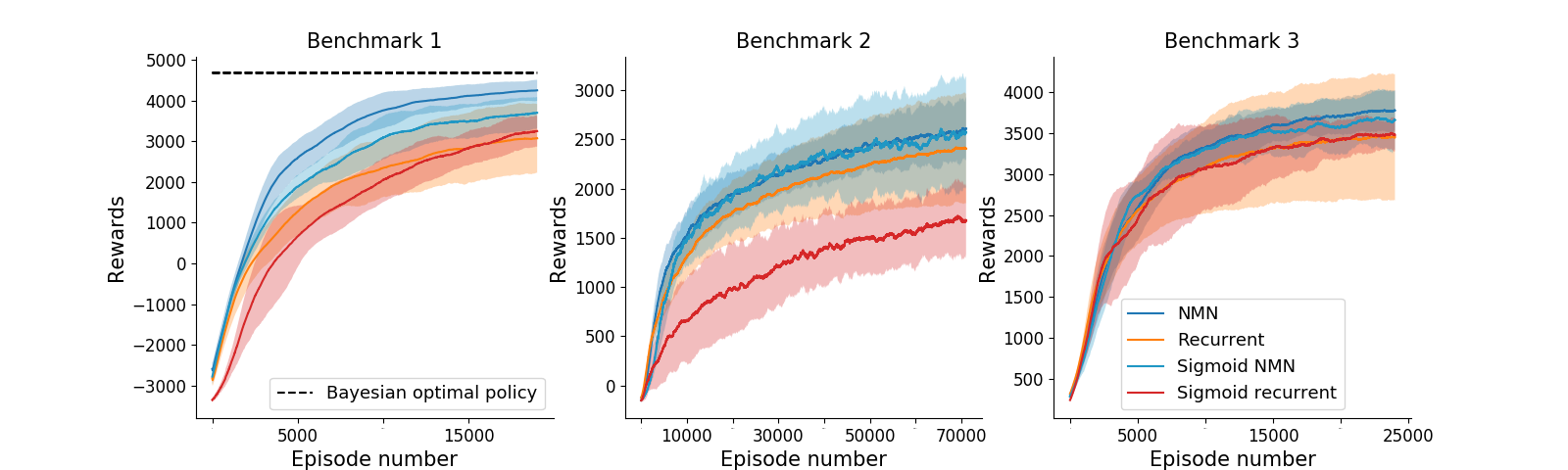}
\caption{Mean ($\pm$ std in shaded) sum of rewards obtained over $15$ training runs with different random seeds with respect to the episode number. Results of benchmark 1,2 and 3 are displayed from left to right. The plots are smoothed thanks to a running mean over $1000$ episodes.}
\label{fig:results_sigmoid}
\end{figure}

\subparagraph{Architecture impact.}
Figure~\ref{fig:results_arch} shows the learning curve, on benchmark 1, for different main network architectures ($0$, $1$ and $4$ hidden layers in the main network respectively). As one can see, RNNs can, in fact, reach NMNs' performances for a given architecture (no hidden layer in this case), but seem relatively dependant on the architecture. On the contrary, NMNs seem surprisingly consistent with respect to the number of hidden layers composing the main network.

\begin{figure}[H]
\centering
\includegraphics[width = 1.0\textwidth]{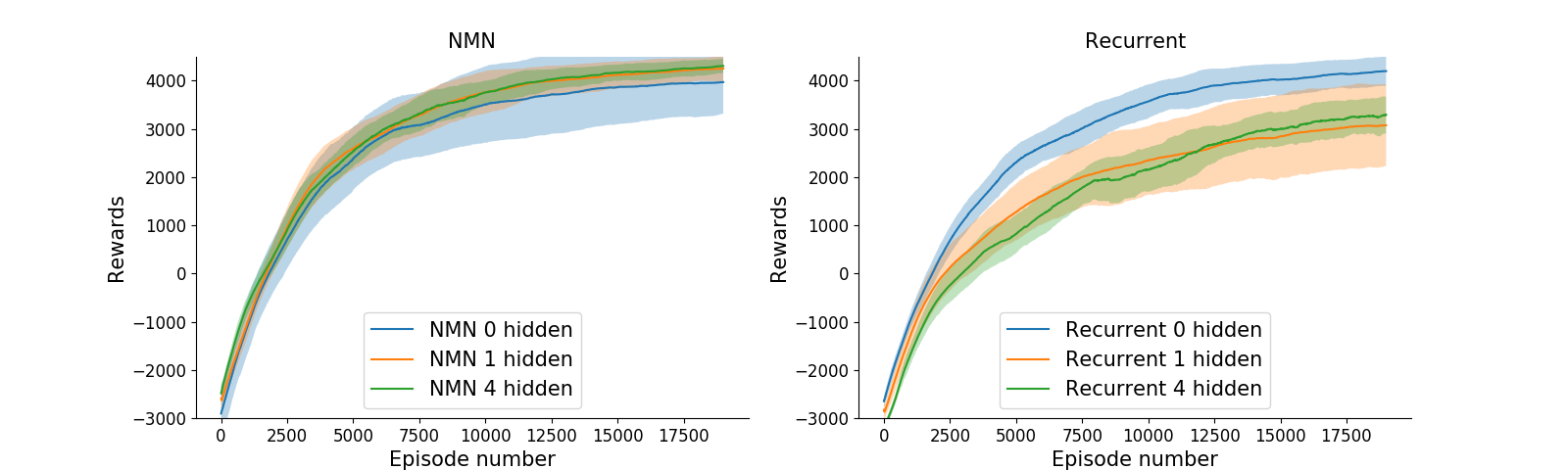}
\caption{Mean ($\pm$ std in shaded) sum of rewards obtained on benchmark 1 over $15$ training runs with different random seeds with respect to the episode number. The plots are smoothed thanks to a running mean over $1000$ episodes.}
\label{fig:results_arch}
\end{figure}

\section{Conclusions}

In this work, we adopt a high-level view of a nervous system mechanism called cellular neuromodulation to improve  artificial  neural  networks  adaptive  capabilities. The results obtained on three meta-RL benchmark problems showed that this new architecture was able to perform better than classical RNN. The work reported in this paper could be extended along several lines. 

First, it would make sense to explore other types of machine-learning problems where adaptation is required. Supervised meta-learning would be an interesting track to follow.

Second, research work could also be carried out to further improve the NMN introduced here. For instance, one could introduce new types of parametric activation functions which are not linear, or even spiking neurons. It would also be of interest to look at sharing activation function parameters per layer. 

Furthermore, it would be a logical progression to tackle other benchmarks to see if the observations made here hold true. More generally, analysing the neuromodulatory signal further in depth (and its impact on activation functions) with respect to different more complex tasks would be worth-while.

Finally, let us emphasize that even if the results obtained by our NMN are good and also rather robust with respect to a large choice of parameters, further research is certainly still needed to better characterise NMNs performances.



\bibliographystyle{alpha}
\bibliography{sample}

\newpage

\section*{Supporting information}

\subsection{Detailed description of benchmark 2 and 3}
\label{app:problems}
Before defining the three benchmark problems, let us remind that for each benchmark, the MDPs that belong to the support of $\mathcal{D}$, which generates the different tasks (see Section~\ref{subsec:formalization}), have transition probabilities and reward functions that differ only according to the value of a scalar $\alpha$. Drawing  an MDP according to $\mathcal{D}$ will amount for all the benchmark problems to draw a value of $\alpha$ according to a probability distribution $P_\alpha(\cdot)$ and to determine the transition function and the reward function that correspond to this value. Let us also denote by $\mathcal{X}$ and $\mathcal{A}$ the state and action spaces respectively.

\subsubsection{Benchmark 2}
\label{app:p2}

\paragraph{State space and action space:}
\[\mathcal{X} = [-3.0,3.0]^2\]
\[\mathcal{A} = \mathbb{R}\]

\paragraph{Probability distribution of $\alpha$:}
\[\boldsymbol{\alpha[i]} \sim \mathbb{U}[-1.0,1.0],~\forall i \in [1,2]\]
\[\boldsymbol{\alpha[3]} \sim \mathbb{U}[-\pi,\pi[\]
\noindent where $\mathbb{U}[a,b]$ stands for a uniform distribution between $a$ and $b$.

\paragraph{Initial state distribution:}\mbox{}\\
The initial state $x_0$ is drawn through $2$ auxiliary random variables that represent the $x$ and $y$ initial coordinates of the agent and are denoted $u^x_{0}, u^y_{0}$. At the beginning of an episode, those variables are drawn as follows:
\[u^k_{0} \sim \mathbb{U}[-1.5 * \pi,1.5*\pi] ~\forall k \in \{x,y\}\]
From those four auxiliary variables, we define $x_0$ as:
\[x_0 = [\boldsymbol{\alpha[1]}-a^x_{0},\boldsymbol{\alpha[2]}-a^y_{0}]\]
The distribution $P_{x_0}(\cdot)$ is thus fully given by the distributions over the auxiliary variables.

\paragraph{Transition function:}\mbox{}\\
Fist, let $target$ be the set of points $(x,y) \in \mathbb{R}^2$ such that \[(x,y) \in target \Leftrightarrow \sqrt{(x-\boldsymbol{\alpha[1]})^2 + (y-\boldsymbol{\alpha[2]})^2} \leq 0.4 \quad.\]
When taking action $a_t$ in state $x_t$ drawing the state $x_{t+1}$ from the transition function amounts to first compute $u^x_{t+1}$ and $u^y_{t+1}$ according to the following procedure:   

\begin{enumerate}
\item If $(u^x_{t}, u^y_{t}) \in target$ then $u^k_{t+1}\sim \mathbb{U}[-1.5,1.5]~\forall k \in \{x,y\}\quad.$
\item If the preceding condition is not met, an auxiliary variable $n_t \sim \mathbb{U}[\frac{-\pi}{4},\frac{\pi}{4}]$ is drawn to compute $u^x_{t+1}$ and $u^y_{t+1}$ through the following sub-procedure:
\begin{enumerate}
\item Step one:
\[u^x_{t+1} = u^x_{t} + 0.25 * (\sin(a_t) + \sin(\boldsymbol{\alpha[3]} + n_t))\]
\[u^y_{t+1} = u^x_{t} + 0.25 * (\cos(a_t) + \cos(\boldsymbol{\alpha[3]} + n_t))\quad.\]

\noindent One can see that taking an action $a_t$ moves the agent in a direction which is the vectoral sum of the intended move $\mb{m}_t$ of direction $a_t$ and of a perturbation vector $\mb{p}_t$ of direction $\alpha + n_t$ sampled through the distribution over $n_t$.
\item Step two: In the case where the coordinates computed by step one lay outside $S[-2; 2]^2$, they are corrected so as to model the fact that when the agent reaches  an edge of the 2D space, it is moved to the opposite edge from which it continues its move. More specifically, $\forall k \in \{x,y\}$:
\[u^k_{t+1} \leftarrow \begin{cases}
u^k_{t+1} - 4 &\text{if}~u^k_{t+1} > 2 \\
u^k_{t+1} + 4 &\text{if}~u^k_{t+1} < -2 \\
u^k_{t+1} &\text{otherwise\quad.}
\end{cases}
\]\end{enumerate}\end{enumerate}

Once $u^x_{t+1}$ and $u^y_{t+1}$ have been computed, $x_{t+1}$ is set equal to $[\boldsymbol{\alpha[1]}-u^x_{t+1},\boldsymbol{\alpha[2]}-u^y_{t+1}]$.
\paragraph{Reward function:}\mbox{}\\
The reward function can be expressed as follows:

\[
\rho(a_t, x_t, x_{t+1}) = \begin{cases}
100 &\text{if} ~ (u^x_{t},u^y_{t}) \in target \\
-2 & \text{otherwise\quad.}
\end{cases}
\]

\subsubsection{Benchmark 3}
\label{app:p3}

\paragraph{State space and action space:}
\[\mathcal{X} = [-2.5,2.5]^4\]
\[\mathcal{A} = \mathbb{R}\]

\paragraph{Probability distribution of $\alpha$:}
\[\boldsymbol{\alpha[i]}
 \sim \mathbb{U}[-1.0,1.0],~\forall i \in [1,2,3,4]\]
\[\boldsymbol{\alpha[5]} \sim \mathbb{U}\{-1,1\}\]

Note that $\boldsymbol{\alpha}[1,2,3,4]$ define the 2-D positions of two targets. For clarity, we will refer to these values respectively by $\alpha^{x_1}$,$\alpha^{y_1}$,$\alpha^{x_2}$ and $\alpha^{y_2}$.

\paragraph{Initial state distribution:}\mbox{}\\
The initial state $x_0$ is drawn through two auxiliary random variables that represent the $x$ and $y$ initial coordinates of the agent and are denoted $u^x_{0}, u^y_{0}$. At the beginning of an episode, those variables are drawn as follows:

\[u^k_{0} \sim \mathbb{U}[-1.5,1.5] ~\forall k \in \{x,y\}\quad.\]
From those six auxiliary variables, we define $x_0$ as:
\[x_0 = [\alpha^{x_1}-u^x_{0},\alpha^{y_1}-u^y_{0},\alpha^{x_2}-u^x_{0},\alpha^{y_2}-u^y_{0}]\quad.\]

\paragraph{Transition function:}\mbox{}\\
For all $i \in \{1,2\}$ let $target_{i}$ be the set of points $(x,y) \in \mathbb{R}^2$ such that \[\sqrt{(x-\alpha^{x_i})^2 + (y-\alpha^{y_i})^2} \leq 0.4\quad.\]. When taking action $a_t$ in state $x_t$, drawing the state $x_{t+1}$ from the transition function amounts to first compute $u^x_{t+1}$ and $u^y_{t+1}$ according to the following procedure:   
\begin{enumerate}
\item If $\exists i \in \{1,2\} : (u^x_{t},u^y_{t}) \in target_i$, which means that the agent is in one of the two targets, then $u^k_{t+1}\sim \mathbb{U}[-1.5,1.5]~\forall k \in \{x,y\}$
\item If the preceding condition is not met, $u^x_{t+1}$ and $u^y_{t+1}$ are computed by the following sub-procedure:
\begin{enumerate}
\item Step one:
\[u^x_{t+1} = u^x_{t} + \sin(a_t * \pi) * 0.25\]
\[u^y_{t+1} = u^y_{t} + \cos(a_t * \pi) * 0.25\quad.\]
\noindent  This step moves the agent in the direction it has chosen.

\item Step two: In the case where the coordinates computed by step one lay outside $[-2; 2]^2$, they are corrected so as to model the fact that when the agent reaches  an edge of the 2D space, it is moved to the opposite edge from which it continues its move. More specifically, $\forall k \in \{x,y\}$:
\[u^k_{t+1} \leftarrow \begin{cases}
u^k_{t+1} - 4.0 &\text{if}~u^k_{t+1} > 2 \\
u^k_{t+1} + 4.0 &\text{if}~u^k_{t+1} < -2 \\
u^k_{t+1} &\text{otherwise\quad.}
\end{cases}\]
\end{enumerate}
\end{enumerate}

Once $u^x_{t+1}$ and $u^y_{t+1}$ have been computed, $x_{t+1}$ is set equal to $[\alpha^{x_1}-u^x_{t+1},\alpha^{y_1}-u^y_{t+1},\alpha^{x_2}-u^x_{t+1},\alpha^{y_2}-u^y_{t+1}]$.
\paragraph{Reward function:}\mbox{}\\
In the case where $(u^x_{t}, u^y_{t})$ either belongs to only $target_1$, only $target_2$ or none of them, the reward function can be expressed as follows:

\[
\rho(a_t, x_t, x_{t+1}) = \begin{cases}
100 * \boldsymbol{\alpha}[5] &\text{if}  ~(u^x_{t}, u^y_{t}) \in target_1 \wedge (u^x_{t}, u^y_{t})\not\in target_2 \\
-50 * \boldsymbol{\alpha}[5] &\text{if}  ~(u^x_{t},u^y_{t}) \in target_2\wedge (u^x_{t}, u^y_{t})\not\in target_1 \\
0 &\text{if}  ~(u^x_{t},u^y_{t}) \not\in target_1\wedge (u^x_{t}, u^y_{t})\not\in target_2\quad.
\end{cases}
\]

In the case where $(u^x_{t}, u^y_{t})$ belongs to both $target_1$ and $target_2$, that is $(u^x_{t}, u^y_{t}) \in target_1 \wedge (u^x_{t}, u^y_{t}) \in target_2$, the reward function can be expressed as follows:

\[
\rho(a_t, x_t, x_{t+1}) = \begin{cases}
100 * \boldsymbol{\alpha}[5] &\text{if}  ~ \sqrt{(u^x_{t}-p^{x_1})^2 + (u^y_{t}-p^{y_1})^2} \leq \sqrt{(u^x_{t}-p^{x_2})^2 + (u^y_{t}-p^{y_2})^2}\\
-50 * \boldsymbol{\alpha}[5] &\text{otherwise\quad.} 
\end{cases}
\]
That is, we consider that the agent belongs to the target to which it is closer to the centre.

\subsection{Advantage actor-critic with generalized advantage estimation}
\label{app:RL}


In our meta-RL setting, both the actor and the critic are parametric functions that are defined on the trajectories' histories. With $\theta \in \Theta$ and $\psi \in \Psi$ the parameters of the actor and critic ($\Theta$ and $\Psi$ are the actor and critic parameters spaces), respectively, we define $\pi_{\theta}$ and $c_{\psi}$ as the policy and critic functions. Let $\pi_{\theta_k}$ and $c_{\psi_{k}}$ be the models for the policy and the critic after $k$ updates of the parameters $\theta$ and $\psi$, respectively. To update from $\theta_k$ to $\theta_{k+1}$ and $\psi_k$ to $\psi_{k+1}$, the actor-critic algorithm uses the policy $\pi_{\theta_k}$ to select actions during $B$ MDPs drawn sequentially from $\mathcal{D}$, where $B \in \mathbb{N}_0$ is a parameter of the actor-critic approach. This interaction between the actor-critic algorithm and the meta-RL problem is presented in a tabular version in Algorithm 1 (\ref{app:tabular}).\\

Using the $L \in \mathbb{N}_0$ first elements of each trajectory generated from the interaction with the $B$ MDPs and the values of $\theta_{k}$ and $\psi_{k}$, the algorithm computes $\theta_{k+1}$ and $\psi_{k+1}$. To this end, the algorithm exploits the set $[h_{B*k,L}, \ldots, h_{B*(k+1)-1,L}]$, which we denote as $H_{k}$. Note that we use a replay buffer for updating $\psi$, thus for this update we also use several previous sets $H_{k-1}$, $H_{k-2}$, etc... A tabular version of the algorithm that details how MDPs are drawn and played, as well as how the set $H_k$ is built, is presented in Algorithm~\ref{alg:run_episodes} of Appendix~\ref{app:tabular}. Let $R^{\pi_{\theta}}_{\mathcal{M}}$ denote the sum of discounted rewards obtained when playing policy $\pi_\theta$ on task $\mathcal{M}$. That is,

\[R^{\pi_{\theta}}_{\mathcal{M}} = \underset{T \rightarrow \infty}{\lim} \sum_{t=0}^{T} \gamma^t r_{t}\]

where $r_t$ are the rewards gathered at each time-step. To have a properly performing actor-critic algorithm, the value chosen for $L$ has to be chosen sufficiently large to produce an accurate estimation of the returns $R^{\pi_{\theta_k}}_{\mathcal{M}_i} ~\forall i \in [B*k,\ldots,B*(k+1)-1]$ obtained by the policy $\pi_{\theta_k}$.

When used in a classical RL setting,
an AC algorithm should interact with its environment to find the value
of $\theta$ that leads to high values of the expected return given a
probability distribution over the initial states. This expected return is written as:
\begin{equation}
\underset{a_t \sim \pi_{\theta} \forall t}{\underset{x_0 \sim P_{x_0}(\cdot)}{\mathbb{E}}} R^{\pi_\theta}_{\mathcal{M}}
\end{equation}
where $\mathcal{M}$ denotes the Markov Decision Process with which the AC algorithm interacts. When working well, actor critic algorithms produce a sequence of policies $\pi_{\theta_1}$, $\pi_{\theta_2}$,
$\pi_{\theta_3}$, $\ldots$ whose expected returns increase as the iterative
process evolves and eventually reaches values close to those obtained by
$\pi_{\theta^*_{\mathcal{M}}}$ with $\theta^*_{\mathcal{M}} = \underset{\theta
  \in \Theta}{\arg \max} \underset{a_t \sim \pi_{\theta} \forall t}{\underset{x_0 \sim P_{x_0}(\cdot)}{\mathbb{E}}}
R^{\pi_\theta}_{\mathcal{M}}$, which, if $\pi^\theta$ is flexible enough, are themselves close to those obtained by an
optimal policy $\pi^*_{\mathcal{M}}$ defined as:
\begin{equation}
\pi^*_{\mathcal{M}} \in \underset{\pi \in \Pi }{\arg \max} \underset{a_t \sim \pi_{\theta} \forall t}{\underset{x_0 \sim P_{x_0}(\cdot)}{\mathbb{E}}} R^{\pi}_{\mathcal{M}}
\end{equation}
where  $\Pi$ is the set of all admissible policies. 

Let $h_t=\{x_0,a_0,r_0,\ldots, x_{t}\}$ be a trajectory generated by
policy $\pi_\theta$ on $\mathcal{M}$ and let $J_{\mathcal{M}}^{\pi_\theta}(h_t)$ be the expected sum of discounted rewards that can be obtained while starting from $h_t$ and playing the policy $\pi_\theta$ in this environment, that is:
\begin{equation}
J^{\pi_\theta}_{\mathcal{M}} (h_t) = \sum_{j=t}^{\infty} \gamma^{j-t}  \rho_\mathcal{M}(x_j, a_j \sim \pi_\theta(h_j), x_{j+1}) \label{eq:returnPolicyFromT} 
\end{equation}
where $\rho_\mathcal{M}(x_j, a_j, x_{j+1})$ is the reward function of task $\mathcal{M}$.
In a classical RL setting, and again for an efficient AC algorithm,
the value of  the critic for  $h_t$, $c_\psi(h_t)$,
also converges to  $J^{\pi_{\theta^*_{\mathcal{M}}}}_{\mathcal{M}} (h_t) $. We also note that in
such a setting, the critic is updated at  iteration $k+1$  in a
direction that provides a better approximation of  $J^{\pi_{\theta_k}}_{\mathcal{M}} (\cdot)$. Now, let us go back to our meta-RL problem and let $V^{\pi} $ denote
the expected sum of returns that policy ${\pi}$ can obtain on this
problem:
\begin{equation}
V^{\pi} =
\underset{\mathcal{M}\sim\mathcal{D}}{\underset{a_t \sim \pi_{\theta} \forall t}{\underset{x_0 \sim P_{x_0}(\cdot)}{\mathbb{E}}}}
\quad .
\end{equation}
Let $\theta^* \in \underset{\theta \in \Theta}{\arg\max} V^{\pi_\theta}$.
When interacting with our meta-RL problem, a performant AC algorithm
should, in principle, converge towards a policy
$\pi_{\hat{\theta}^*}$, leading to a value of
$V^{\pi_{\hat{\theta}^*}}$ close to $V^{\pi_\theta^*}$ that is itself
close to $\underset{\pi  \in \Pi}{\max} V^{\pi}$. A policy $\pi^*$ such that $\pi^* \in \underset{\pi \in \Pi}{\arg \max} V^\pi $ is called a Bayes optimal policy in a Bayesian RL setting where the distribution $\mathcal{D}$ is assumed to be known.  If we are working
with policies that are, indeed, able to quickly adapt to  the environment,  we  may  also  expect that  the  policy
$\pi_{\hat{\theta}^*}$ learned by  the algorithm  is such  that, when
applied on an $\mathcal{M}$ belonging to the support of $\mathcal{D}$, it
leads to a value of  $J^{\pi_{\hat{\theta}^*}}_{\mathcal{M}} (h_t)$ close to $
\underset{\pi \in \Pi}{\max}J^{\pi}_{\mathcal{M}}  (h_t) $ as $t$
increases. In other words, once the agent has gathered enough
information to adapt to the current MDP, it should start behaving
(almost) optimally. This is the essence of meta-RL.

We may also expect that, in such case, the value of the critic for
$h_t$ when the budget is exhausted closely estimates
the expected value of the future discounted rewards that can be
obtained when using policy $\pi^{\hat{\theta}^*}$ and after
having already observed a trajectory $h_t$. Therefore, we may also expect that once the episode budget is exhausted,
$c_{\psi}(h_t)$:
\begin{enumerate}
\item will  be close to  $\underset{\mathcal{M} \sim \mathcal{D}}{\mathbb{E}} J^{\pi_{\hat{\theta}^*}}_{\mathcal{M}} (h_t)    \simeq  \underset{\mathcal{M} \sim \mathcal{D}}{\mathbb{E}} \underset{\pi \in \Pi}{\max} J^{\pi}_{\mathcal{M}} (h_t) $                                     if   $h_t =\{x_0 \}$;
\item will, as $t$ increases, tend to get closer to $\underset{\pi \in \Pi}{\max} J^{\pi}_{\mathcal{M}} (h_t)    \simeq J^{\pi_{\hat{\theta}^*}}_{\mathcal{M}} (h_t)$ where $\mathcal{M}$ can be any environment belonging to the support of $\mathcal{D}$ used to generate $h_t$.   
\end{enumerate}

Existing actor-critic algorithms mainly differ from each other by the way the actor and critic are updated. While in early actor-critic algorithms the critic was directly used to compute the direction of update for the actor's parameters (see for example the REINFORCE policy updates \cite{reinforce}), now it is more common to use an advantage function. This function represents the advantage in terms of return of selecting specific actions given a trajectory history (or simply a state when AC algorithms are used in a standard setting) over selecting them following the policy used to generate the trajectories. Here, we use generalised advantage estimations (GAE), as introduced in \cite{GAE}. More recently, it has been shown that avoiding too large policy changes between updates can greatly improve learning (\cite{TRPO}, \cite{PPO}). Therefore, while in classical AC algorithms the function used to update the actor aims at representing directly the gradient of the actor's return with respect to its parameters, we rather update the actor's parameters $\theta$ by minimising a loss function that represents a surrogate objective. We have selected as surrogate function one that is similar to the one introduced in \cite{PPO} with an additional loss term that proved to improve (albeit slightly) the performances of PPO in all cases. \\

As our actor and critic are modelled by differentiable functions, they are both updated through gradient descent. We now proceed to explain the losses use to compute the gradient for both the actor and the critic.  

\paragraph{Actor update}
First, we define the temporal error difference term for any two consecutive time-steps of any trajectory: 

$$TD_{i} = r_{i} + \gamma * c_{\psi_{k}}(h_{i+1}) - c_{\psi_{k}}(h_{i}), ~\forall i \in [0,\ldots,L]$$

where $\psi_k$ denotes the critic's parameters for playing the given trajectory. This temporal difference term represents, in some sense, the (immediate) advantage obtained, after having played action $a_{j}$ over what was expected by the critic. If $c_{\psi_k}(\cdot)$ was the true estimate of $J^{\pi_{\theta_{k}}}_{\mathcal{M}}(\cdot)$ and if the policy played was $\pi_{\theta_{k}}$, the expected value of these temporal differences would be equal to zero. We now define the GAE's terms that will be used later in our loss functions:  

\begin{equation}\label{eq:gae}GAE^i_{j} = \sum_{t = j}^{L} (\gamma * \lambda)^{t-j} * TD^j_{i}, ~\forall j \in [1,\ldots,E], i \in [0,\ldots,L']\end{equation}

\noindent where $\lambda \in [0,1]$ is a discount factor used for computing GAEs, $TD^j_i$ is the value of $TD_i$ for trajectory $j$, and where $L'$ is another hyper-parameter of the algorithm, chosen in combination with $L$ in order to have a value of $GAE_{i,j}$ that accurately approximates $\sum_{t=j}^{\infty} (\gamma * \lambda)^{k-j} * TD^j_{i}$ $\forall i,j$. Note that the value chosen for $L'$ also has to be sufficiently large to provide the loss function with a sufficient number of GAE terms. These GAE terms, introduced  in \cite{GAE}, represent the exponential average of the discounted future advantages observed. Thanks to the fact that GAE terms can catch the accumulated advantages of a sequence of actions rather than of a single action, as it is the case with the temporal difference terms, they can better represent the advantage of the new policy played by the AC algorithm over the old one (in terms of future discounted rewards).  \\

In the loss function, we will actually not use the advantage terms as defined by Equation~\ref{eq:gae}, but normalised versions in order to have advantages that  remain in a similar range regardless of
rewards magnitude. Thanks to this normalisation, the policy learning rate does not have to be tuned according to the loss magnitude. However, this normalisation does not mask actions that have led to higher
or lower returns than average. We normalize as follows $\forall k \in [1,\ldots,E]$ with $E$ the number of actor and critic updates that have been carried: 

\[\mu_{gae} = \avsum_{j=0}^{B-1}[{\avsum_{i=0}^{L'-1} GAE^{B*k+j}_{i}}]\]

$$\sigma_{gae} = \sqrt{\avsum_{j=0}^{B-1}[\avsum_{i=0}^{L'-1} (\mu_{gae} - GAE^{B*k+j}_{i})^2]}$$
 
$$GAE^j_{i} = \frac{GAE^{B*k+j}_{i}-\mu_{gae}}{\sigma_{gae}}~\forall (j,i) \in ([0,\ldots,B-1] * [0,\ldots,L'-1])$$

\noindent where $\avsum$ is the symbol we use to represent the average sum operator (i.e. $\avsum_{x=1}^{m}f(x) = \sum_{x=1}^m \frac{f(x)}{m}$). To define the loss functions used to compute $\theta_{k+1}$ and $\psi_{k+1}$, only the GAE terms corresponding to time-steps $[0,\ldots,L']$ of episodes $[B*k, \ldots, B*(k+1) -1]$ are computed. A tabular version of the algorithm used to compute these terms is given in Algorithm~\ref{alg:ac_procedure} of Appendix~\ref{app:tabular} \footnote{Although not explicitly written in the text for clarity, we use a normalisation technique when computing discounted sums for the AC algorithm update. In fact, when carrying an update of the AC algorithm, if rewards appear in discounted sums, they are multiplied by $(1-\gamma)$. This has for effect that the discounted sum values remain of the same magnitude regardless of $\gamma$. The implications of this normalization are two-fold. (i) The critic does not directly approximate $J^\pi_{\mathcal{M}}(\cdot)$ but rather $(1-\gamma)*J^\pi_{\mathcal{M}}(\cdot)$. (ii) Second, for the temporal differences to remain coherent with this normalisation, $r_i$ must also be multiplied by $(1-\gamma)$ when computing $TD_i$. Those two small changes are included in Algorithm~\ref{alg:ac_procedure}.}. \\

Once advantages have been computed, the values of $\theta_{k+1}$ are computed using updates that are strongly related to PPO updates with a Kullback Leibler (KL) divergence implementation \cite{PPO}. The loss used in PPO updates is composed of two terms: a classic policy gradient term and a penalisation term. Let us now present a standard policy gradient loss, note that from now on, we will refer to the value of $a_t$, $x_t$ and $h_t$ at episode $i$ by $a^i_t$, $x^i_t$ and $h^i_t$ respectively.

\begin{equation}
\label{eq:lvanilla}
\mathcal{L}_{vanilla}(\theta) = -\avsum_{(i,t) \in \mathcal{B}_k}{}\frac{\pi_{\theta}(a^i_{t}|h^i_{t})}{\pi_{\theta_{k}}(a^i_{t}|h^i_{t})} * GAE^i_{t} \quad.\end{equation}

\noindent where $\mathcal{B}_k$ is the set of all pairs $(i,t)$ for which $i,t \in ([B*k,\ldots,B*(k+1)-1] *[0,\ldots,L'])$, that is, the set containing the first $L'$ time-steps of the $B$ trajectories played for iteration $k$ of the actor-critic algorithm. 

One can easily become intuitive about Equation~\ref{eq:lvanilla} as, given an history $h^i_{t}$, minimising this loss function tends to increase the probability of the policy taking actions leading to positive advantages (i.e. $GAE^i_{t} > 0$) and decreases its probability to take actions leading to negative advantages (i.e. $GAE^i_{t} < 0$).
It has been found that to obtain good performances with this above-written loss function, it was important to have a policy that does not change too rapidly from one iteration to the other. Before explaining how this can be achieved, let us first give an explanation on why it may be important to have slow updates of the policy. Let us go back to the loss function given by Equation~\ref{eq:lvanilla}. Minimising this loss function will give a value for $\theta_{k+1}$ that will lead to higher probabilities of selecting actions corresponding to high values of the advantages $GAE^i_{t}$. A potential problem is that these advantages are not really related to the advantages of the would-be new policy $\pi_{\theta_{k+1}}$  over $\pi_{\theta_k}$ but are instead related to the advantages of policy $\pi_{\theta_{k}}$  over $\pi_{\theta_{k-1}}$. Indeed, the advantages $GAE^i_{t}$ are computed using the value function $c_{\psi_k}$, whose parameters have been updated from $\psi_{k-1}$ in order to better approximate the sum of discounted rewards obtained during the episodes $[B*(k-1),\ldots,B*k-1]$. It clearly appears that $\psi_k$ has, in fact, been updated to approximate discounted rewards obtained through the policy $\pi_{\theta_{k-1}}$ (used to play episodes for update $k-1$). A solution to this problem is to constraint the minimisation to reach a policy $\pi_{\theta_{k+1}}$ that does not stand too far from $\pi_{\theta_k}$. We may reasonably suppose that the advantage function used in (\ref{eq:lvanilla}) still correctly reflects the real advantage function of $\pi_{\theta_{k+1}}$ over $\pi_{\theta_k}$. To achieve this, we add a penalisation term $\mathcal{P}(\theta)$ to the loss function. In the PPO approach, the penalisation term is $\mathcal{P}_{ppo}(\theta) = \beta_k * d(\theta)$, where:
\begin{itemize}
    \item[i)] $\beta_k$ is an adaptive weight
    \item[ii)] $d(\theta) = \avsum_{[i,t] \in \mathcal{B}_k}{}[KL(\pi_{\theta_k(.|h_{i,t})},\pi_{\theta(.|h_{i,t})})]$, where $KL$ is the Kullback-Leibler divergence, detailed later on. This term penalises policies that are too different from $\pi_{\theta_k}$.
\end{itemize} 





We note that the $\beta_k$ dynamical updates use a hyper-parameter $d_{targ} \in \mathbb{N}_0$ called the divergence target. The update is done through the following procedure (note that, unlike updates of $\beta$ proposed in \cite{PPO}, we constrain $\beta$ to remain in the range $[\beta_{min}, \beta_{max}]$; we explain later why):

\begin{equation}\label{eq:beta_update}\beta_{k+1} = \begin{cases}
    \max(\beta_{min},\frac{\beta_k}{1.5}) &\text{if}~d(\theta) < \frac{d_{targ}}{2.0}\\
    \min(\beta_{max}, \beta_k * 1.5) &\text{if}~d(\theta) > d_{targ} * 2\\
    \beta_k ~&\text{otherwise} \quad .
    \end{cases}
\end{equation}

With this update strategy, the penalisation term will tend to evolve in a way such that the KL divergence between two successive policies does not tend to go beyond $d_{targ}$ without having to add an explicit constraint on $d$, as was the case in Trust Region Policy Optimization (TRPO) updates \cite{TRPO}, which is more cumbersome to implement. \\

As suggested in \cite{blog}, adding another penalisation term (squared hinge loss) to $\mathcal{P}_{PPO}$ to further penalise the KL divergence, in cases where it surpasses $2*d_{targ}$, improved algorithm performance. The final expression of the penalisation term is:

\[\mathcal{P}(\theta) = \beta_k * d(\theta) + \delta * \max(0,d(\theta) - 2*d_{targ})^2\]

\noindent where $\delta$ is a hyper-parameter that weights the third loss term. The loss function $\mathcal{L}_{policy}$ that we minimise as a surrogate objective becomes:

\begin{equation}
\label{eq:lpolicy}
\mathcal{L}_{policy}(\theta) = \mathcal{L}_{vanilla}(\theta) + \mathcal{P}(\theta)\end{equation}

We now detail how to compute the KL divergence. First, let us stress that we have chosen
to work with multi-variate Gaussian policies for the actor. This choice is particularly
well suited for MDPs with continuous action spaces. The approximation architecture of
the actor will therefore not directly output an action, but the means and standard
deviations of an m-dimensional multi-variate Gaussian from which the actor’s policy can be defined in a straightforward way. For each dimension, we bound the multi-variate Gaussian to the support, $\mathcal{U}$, by playing the action that is clipped to the bounds of $\mathcal{U}$ whenever the multi-variate Gaussian is sampled outside of $\mathcal{U}$. In the remaining of this paper, we will sometimes abusively use the terms "output of the actor at time $t$ of episode $i$" to refer to the means vector $\mu_{i,t}^{\theta_k}$ and the standard deviations vector $\sigma_{i,t}^{\theta_k}$ that the actor uses to define its probabilistic policy at time-step $t$ of episode $i$. Note that we have chosen to work with a diagonal covariance matrix for the multi-variate Gaussian distribution. Its diagonal elements correspond to those of the vector $\sigma_{i,t}^{\theta_k}$. We can then compute the KL divergence in each pair $[i,t]$ following the well-established formula:
\begin{multline}
KL(\pi_{\theta_k}(\cdot|h^i_{t}),\pi_{\theta}(\cdot|h^i_{t})) = \\ \frac{1}{2}\{tr(\Sigma^{-1}_{\theta,i,t} \Sigma_{\theta_k,i,t}) + (\mu^{\theta}_{i,t} -\mu^{\theta_k}_{i,t})^T\Sigma^{-1}_{\theta,i,t}(\mu^{\theta}_{i,t} - \mu^{\theta_k}_{i,t}) - k + \ln(\frac{|\Sigma_{\theta,i,t}|}{|\Sigma_{\theta_k,i,t}|})\}\end{multline}
where $\Sigma_{\theta_k,i,t}, \Sigma_{\theta,i,t}$ are the diagonal covariance matrices of the two multi-variate Gaussian distributions $\pi_{\theta_k}(\cdot|h^i_t),\pi_{\theta}(\cdot|h^i_t)$ that can be derived from $\sigma_{i,t}^{\theta_{k}}$ and $\sigma_{i,t}^{\theta}$. The loss function $\mathcal{L}_{vanilla}$ can be expressed as a function of $\Sigma_{\theta_k,i,t}$, $\Sigma_{\theta,i,t}$, $\mu^{\theta}_{i,t}$ and $\mu^{\theta_k}_{i,t}$ when working with a multi-variate Gaussian. To this end, we use the log-likelihood function $\ln{(\pi_\theta(a_{i,t}|h^i_t))}$, which gives the log-likelihood of having taken action $a^i_t$ given a trajectory history $h^i_t$. In the case of a multi-variate Gaussian, $\ln{(\pi_\theta(a^i_t|h^i_t))}$ is defined as:
\begin{equation}\label{eq:ln}\ln{(\pi_\theta(a^i_t|h^i_t))}  =-\frac{1}{2}(\ln(|\Sigma_{\theta,i,t}|) + (a^i_t-\mu^{\theta}_{i,t})^T * \Sigma_{\theta,i,t}^{-1}*(a^i_t-\mu^{\theta}_{i,t}) + m * \ln(2*\pi))\end{equation}

\noindent where $m$ is the dimension of the action space and where $|\Sigma_{\theta,i,t}|$ represents the determinant of the matrix. From this definition, one can rewrite $\mathcal{L}_{vanilla}$ as:
\begin{equation}\label{eq:vanilla_fully_expressed}
\mathcal{L}_{vanilla} = -\avsum_{[i,t] \in \mathcal{B}_k} e^{\ln{(\pi_{\theta}(a^i_t|h^i_t))}-\ln{(\pi_{\theta_k}(a^i_t|h^i_t))}} * GAE^i_{t}\quad .
\end{equation}

\noindent By merging equation~(\ref{eq:vanilla_fully_expressed}), (\ref{eq:ln}) and equation~(\ref{eq:lpolicy}), one gets a loss $\mathcal{L}_{policy}$ that depends only on $\Sigma_{\theta_k,i,t}$, $\Sigma_{\theta,i,t}$, $\mu^{\theta}_{i,t}$ and $\mu^{\theta_k}_{i,t}$.

\paragraph{Critic update}
The critic is updated at iteration $k$ in a way to better approximate the expected return obtained when following the policy $\pi_{\theta_k}$, starting from a given trajectory history. To this end, we use a mean-square error loss as a surrogate objective for optimizing $\psi$. First, we define $\hat{R}^i_{j} = \sum_{k=j}^{L} \gamma^{k-j}*r^i_{j}$ $\forall i,j \in [B*k,\ldots,B*(k+1)-1],[0,\ldots,L]$.
 From the definition of $\hat{R}^i_{j}$ we express the loss as:
    \begin{equation}
\label{eq:lcritic}
\mathcal{L}_{critic}(\psi) = \sum_{(i,t) \in \mathcal{B}_{k-CRB}}^{}[(c_{\psi}(h^i_{t}) - \hat{R}^i_{t})^2]\end{equation}
where (i) $crb \in \mathbb{N}_0$ is a hyper-parameter; (ii) $\mathcal{B}_{k-crb}$ is the set of all pairs $(i,t)$ for which $i,j \in ([B*(k-crb),\ldots,B*(k+1)-1]*[0,\ldots,L'])$. The set $\mathcal{B}_{k-crb}$ used in (\ref{eq:lcritic}) contains all the pairs from the current trajectory batch and from the $crb$ previous trajectory batches. We call this a replay buffer whose length is controlled by $crb$ which stands for "\textbf{c}ritic\textbf{r}eplay\textbf{b}uffer". Minimising $\mathcal{L}_{critic}$ does not lead to updates such that $c_{\psi}$ directly approximates the average expected return of the policy $\pi_{\theta_k}$. Rather, the updates are such that $c_\psi$ directly approximates the average expected return obtained by the last $crb+1$ policies played. We found out that using a replay buffer for the critic smoothed the critic's updates and improved algorithm performances.

Note that the loss (\ref{eq:lcritic}) is only computed on the $L' << L$ first time-steps of each episode, as was the case for the actor. The reason behind this choice is simple. The value function $c_{\psi_k}$ should approximate $R^i_{j} = \sum_{t = j}^{+\infty}\gamma^{t-j}*r^i_{t}$ for every $h^i_{j}$, where $R^i_{j}$ the infinite sum of discounted rewards that are attainable when "starting" from $h^i_{j}$. However, this approximation can become less accurate when $j$ becomes close to $L$ since we can only guarantee $\hat{R}^i_{j}$ to stand in the interval: $[R^i_{j} - \frac{\gamma^{L-j}}{1-\gamma} R_{max} , R^i_{j} - \frac{\gamma^{L-j}}{1-\gamma} R_{min}]$. Hence this choice of $L'$. \\

\paragraph{Gradients computation and update}
The full procedure is available as Tabular versions in \ref{app:tabular}. As a summary, we note that both the actor and critic are updated using the Adam procedure (\cite{adam}) and back-propagation through time (\cite{bptt}). The main difference between both updates is that the actor is updated following a full-batch gradient descent paradigm, while the critic is updated following a mini-batch gradient descent paradigm.


\subsubsection{Tabular version}
\label{app:tabular}

\begin{algorithm}[H]
\caption{Advantage actor-critic with generalised advantage estimate for solving the meta-RL problem}
\label{alg:ac_top_level}
\begin{algorithmic}[1]
\State \textbf{Run}({$\mathcal{D}$, $E$, $\mathbf{hyperparameters_0}$})
\State \textbf{Inputs}: 
\begin{enumerate}[leftmargin=1.5cm,labelsep=0cm,align=left,label={[\arabic*]}]
\item $\mathcal{D}$ : The distribution over MDPs.
\item $E$ : The total episodes budget.
\item $hyperparameters_0$ : The  set of hyper-parameters that contains the following elements: 
\begin{itemize}[label={\tiny\raisebox{1ex}{\textbullet}}]
\item $B$ : Number of episodes played between updates.
\item $P_{\theta_0}$ and $P_{\psi_0}$ : The distributions for initialising actor and critic's parameters. Those distributions are intrinsically tied to the models used as function approximators.
\item $\lambda \in [0,1]$ : The discount factor for computing GAE.
\item $L$ : Number of time steps played per episode.
\item $L'$ : Number of time steps per episode used to compute gradients.
\item $e_{a}$ : The number of epochs per actor update.
\item $e_{c}$ : The number of epochs per critic update.
\item $\eta$ : The squared hinge loss weight.
\item $d_{targ}$ : The KL divergence target.
\item $d_{thresh}$ : The threshold used for early stopping.
\item $\beta_{min}$ and $\beta_{max}$ : The minimum and maximum $\beta$ values.
\item $\beta_0$ : The initial value of $\beta_k$ for penalising the KL divergence.
\item $a_{lr_0}$ : The initial value of the policy learning rate $a_{lr_k}$.
\item $c_{v_0}$, $c_{z_0}$, $a_{v_0}$ and $a_{z_0}$ : The initial value for the ADAM optimiser moments $c_{v_k}$, $c_{z_k}$, $a_{v_k}$ and $a_{z_k}$.
\item $\epsilon$, $\omega_1$, $\omega_2$ : The three ADAM optimiser hyper-parameters.
\item $c_{lr}$ : The critic learning rate.
\item $cmb$ : The mini-batch size used for computing the critic's gradient.
\item $crb$ : The number of previous trajectory batches used in the replay buffer for the critic.
\end{itemize}
\end{enumerate}

We note that some of the hyper-parameters are adaptive. These are $\beta_k$, $a_{lr_k}$, $c_{v_k}$, $c_{z_k}$, $a_{v_k}$ and $a_{z_k}$. Thus the hyper-parameter vector may have to change in between iterations. For this reason we introduce the notation $hp_k$ which represents the hyper-parameter vector with the values of the adaptive parameters at iteration $k$.
\State $hp_0 \leftarrow hyperparameter_0$
\State $k \leftarrow 0$ 
\State $\theta_0 \sim P_{\theta_0}(.)$ \Comment Random initialisation
\State $\psi_0 \sim P_{\psi_0}(.)$ \Comment Random initialisation
\While {$B * k < E$}
\State $H_k = \textbf{run episodes}(k, \theta_{k}, \mathcal{D}, hp_k)$
\State $\theta_{k+1},\psi_{k+1}$= \textbf{update ac}($H_{\max(0,k-CRB)},\ldots,H_{k}$,$\theta_k$,$\psi_k$,$hp_k$)
\State $k \leftarrow k + 1$
\EndWhile
\end{algorithmic}
\end{algorithm}

\begin{algorithm}[H]
\caption{Kth run of $B$ episodes}
\label{alg:run_episodes}
\begin{algorithmic}[1]
\State \textbf{run episodes}({$k$, $\theta_k$, $\mathcal{D}$, $hp_k$})
\State \textbf{Inputs}: 
\begin{enumerate}[leftmargin=1.5cm,labelsep=0cm,align=left,label={[\arabic*]}]
\item $\theta_k$ : The parameters of the policy at iteration $k$.
\item $\mathcal{D}$ : The distribution from which the MDPs are sampled.
\item $hp_k$ : In this procedure, we use as hyper-parameters:
\begin{itemize}[label={\tiny\raisebox{1ex}{\textbullet}}]
\item $B$ : The number of episodes to be played.
\item $L$ : The number of time steps played by episode.
\end{itemize}
\end{enumerate}
\State \textbf{Output}: 
\begin{enumerate}[leftmargin=1.5cm,labelsep=0cm,align=left,label={[\arabic*]}]
\item $H^k$ : The set of $B$ trajectories $[h^{B*k}_{L}, h_{B*(k+1),L}, \ldots, h^{B*(k+1)-1}_{L}]$ played during this procedure.
\end{enumerate}

\State $i \leftarrow B * k$ 
\While {$i < B * k + B$}
\State $t \leftarrow 0$
\State $\mathcal{M} \sim \mathcal{D}$
\State $x^i_{t} \sim P_{x_0}(.)$
\State $h^i_{t} = [x^i_{t}]$
\While {$t <= L$}
\State $a^i_{t} \sim \pi_{\theta_k}(h^i_{t})$
\State $x^i_{t+1} \sim P^{\mathcal{M}}(x^i_{t+1}|x^i_{t},a^i_{t})$ 
\Comment The right-side refers to $P(x^i_{t+1}|x^i_t,a^i_t)$ of current task $\mathcal{M}$.
\State $r^i_{t} = \rho^{\mathcal{M}}(x^i_{t}, a^i_{t}, x^i_{t+1})$
\Comment The right-side refers to $\rho(x_{t+1}, x_t, a_t)$ of $\mathcal{M}$.
\State $h^i_{t} = [x^i_{0},a^i_{0},r^i_{0},\ldots,x^i_{t}]$
\State $t \leftarrow t + 1$
\EndWhile
\State $i \leftarrow i + 1$
\EndWhile
\State \textbf{Return }~$H_k = [h^{B*k}_{L},\ldots,h^{B*(k+1)-1}_L]$
\end{algorithmic}
\end{algorithm}

\begin{algorithm}[H]
\caption{Kth update of the actor critic model}
\label{alg:ac_procedure}
\begin{algorithmic}[1]
\State \textbf{update ac}({$H^{k-crb},\ldots,H^{k}$, $\theta_k$, $\psi_k$, $hp_k$})
\State \textbf{Inputs}: 
\begin{enumerate}[leftmargin=1.5cm,labelsep=0cm,align=left,label={[\arabic*]}]
\item $H^{k-crb},\ldots,H^{k}$ : The $crb+1$ last sets of $B$ trajectories of length $L$.
\item $\theta_k$ and $\psi_k$ : The parameters of the actor and critic after $k$ updates.
\item $hp_k$ : In this procedure, we use as hyper-parameter:
\begin{itemize}[label={\tiny\raisebox{1ex}{\textbullet}}]
\item $\lambda \in [0,\ldots,1]$ : The discount factor for computing GAE.
\item $a_{lr_k}$ : The current policy learing rate.
\item $\beta_k$ : The current KL divergence penalisation.
\end{itemize}
\end{enumerate}
\State \textbf{Output}: 
\begin{enumerate}[leftmargin=1.5cm,labelsep=0cm,align=left,label={[\arabic*]}]
\item $\theta_{k+1}$, $\psi_{k+1}$ : The updated actor and critic parameters.
\end{enumerate}

\State $i \leftarrow B * k$
\State $\mb{D} \leftarrow \emptyset$
\While {$i < B * k + B$}
\State $D^i_{j} = \sum_{t = j}^{L}{\gamma ^ {t-j} * r_{j} * (1 - \gamma)},~\forall j \in [0,\ldots,L-1]$
\State $\mb{D} \leftarrow \mb{D} \cup D^i_{j}$
\State $TD_{j} = (1-\gamma) * r_{j} - c_{\psi_k}(h_{j}) + c_{\psi_k}(h_{j+1}), j \in [0,\ldots,L-1]$
\State $GAE^i_{j} = \sum_{t = j}^{L} (\gamma * \lambda)^{t-j} * TD_{j}, ~\forall j \in [0,\ldots,L-1]$
\EndWhile
\State $\mu = \avsum_{i=B*k}^{B*(k+1)-1}\avsum_{j=0}^{L-1} GAE^i_{j}$
\State $\sigma = \sqrt{\avsum_{i=B*k}^{B*(k+1)-1}\avsum_{j=0}^{L-1} (\mu - GAE^i_{j})^2}$
 
\State $GAE^i_{j} \leftarrow \frac{GAE^i_{j}-\mu}{\sigma}~\forall i \in [B*k,\ldots,B*(k+1)-1], j \in [0,\ldots,L-1]$
\State $\mb{A} = [GAE^{B*k+i}_j, ~\forall (i,j) \in ([0,\ldots,B-1] * [0,\ldots,L])]$

\State $\theta_{k+1}$ = \textbf{update policy parameters}($H^k$, $\mb{A}$, $\theta_k$, $hp_k$)

\State $\psi_{k+1}$= \textbf{update critic parameters}($H^{k-crb},\ldots,H^{k}$, $\mb{D}$, $\psi_k$, $hp_{k}$)
\State \textbf{Return} $\theta_{k+1}, \psi_{k+1}$

\end{algorithmic}
\end{algorithm}


\begin{algorithm}[H]
\caption{Update from $\theta_k$ to $\theta_{k+1}$}
\label{alg:actor_math_update}
\begin{algorithmic}[1]
\State \textbf{update policy parameters}({$H_{k}$, $\mb{A}$,$\theta_k$,$hp_k$})
\State \textbf{Inputs}: 
\begin{enumerate}[leftmargin=1.5cm,labelsep=0cm,align=left,label={[\arabic*]}]
\item $H^{k}$ : The set of $B$ trajectories of length $L$.
\item $\theta_{k}$ : The actor's parameters.
\item $\epsilon$, $\omega_1$ and $\omega_2$ : The three ADAM optimizer hyper-parameters.
\item $hp_k$ : In this procedure, we use as hyper-parameters:
\begin{itemize}[label={\tiny\raisebox{1ex}{\textbullet}}]
\item $e_{actor}$ : The number of epochs per actor update.
\item $\eta$ : The squared hinge loss weight.
\item $d_{targ}$ : The KL divergence target.
\item $d_{thresh}$ : The threshold used for early stopping.
\item $L'$ : The number of time-steps per trajectory used for computing gradients.
\item $\beta_k$ : The KL penalisation weight.
\item $a_{lr_k}$ : The actor learning rate.
\item $a_{v_k}$ and $a_{z_k}$ : The last ADAM moments computed at iteration $k-1$.
\end{itemize}
\end{enumerate}
\State \textbf{Output}: 
\begin{enumerate}[leftmargin=1.5cm,labelsep=0cm,align=left,label={[\arabic*]}]
\item $\theta_{k+1}$ : The updated actor parameters.
\end{enumerate}

\State $m \leftarrow 0$
\State $\mathcal{B} \leftarrow [(B*k+i,j), ~\forall (i,j) \in ([0,\ldots,B-1] * [0,\ldots,L'])]$
\State $\theta' \leftarrow \theta_k$
\State $a'_{v} \leftarrow a_{v_{k}}$
\State $a'_{z} \leftarrow a_{z_{k}}$
\While {$m < AE$}

\State $\mathcal{L}_{vanilla} = - \avsum_{(i,t) \in \mathcal{B}}\frac{\pi_{\theta}(a^i_{t}|h^i_{t})}{\pi_{\theta_{k}}(a^i_{t}|h^i_{t})}*GAE^i_{t}$

\State $d = \avsum_{(i,t) \in \mathcal{B}} KL(\pi_{\theta_k}(.|h^i_{j}),\pi_{\theta}(.|h^i_{j}))$

\State $s = [\max(0, (d-2*d_{targ}))]^2$

\State $\mathcal{L}_{policy} = \mathcal{L}_{vanilla} + \beta_k * d + \eta * s$

\State $\nabla_{\theta}\mathcal{L}_{policy}(\theta') = \textbf{compute gradients}(\mathcal{L}_{policy}, \mathcal{B}_k, \theta')$

\State $a'_{lr} = a_{lr_k} * \frac{\sqrt{1-\omega^{k*e_{actor}+m}_2}}{1-\omega^{k*e_{actor}+m}_1}$
\State $a'_{z} \leftarrow \omega_1 * a'_{z} + (1-\omega_1)*\nabla_{\theta}\mathcal{L}_{policy}(\theta')$
\State $a'_v \leftarrow \omega_2 * a'_v + (1-\omega_2)*\nabla_{\theta}\mathcal{L}_{policy}(\theta')\odot\nabla_{\theta}\mathcal{L}_{policy}(\theta')$
\State $\theta' \leftarrow \theta' - \frac{a'_{lr}*a'_z}{\sqrt{a'_v} +\epsilon}$
\State $m \leftarrow m + 1$
\If{$d > d_{threshold} * d_{targ}$}
\Comment Early stop
\State $\theta' \leftarrow \theta_k$
\State $m \leftarrow e_{actor}$
\EndIf
\EndWhile
\State \textbf{update auxiliary parameters}$(d,hp_k)$
\State $a_{v_{k+1}} \leftarrow a'_v$
\State $a_{z_{k+1}} \leftarrow a'_z$
\State $\theta_{k+1} \leftarrow \theta'$
\State \textbf{Return} $\theta_{k+1}$

\end{algorithmic}
\end{algorithm}

\begin{algorithm}[H]
\caption{Actor auxiliary parameters update}
\label{alg:update_aux}
\begin{algorithmic}[1]
\State \textbf{update auxiliary parameters}(d,$hp_k$)
\State \textbf{Inputs}: 
\begin{enumerate}[leftmargin=1.5cm,labelsep=0cm,align=left,label={[\arabic*]}]
\item $d$ : The KL divergence between $\pi_{\theta_k}$ and $\pi_{\theta_{k+1}}$ empirically averaged.
\item $hp_k$ : In this procedure, we use as hyper-parameters:
\begin{itemize}[label={\tiny\raisebox{1ex}{\textbullet}}]
\item $d_{targ}$ : The KL divergence target.
\item $\beta_{min}$ and $\beta_{max}$ : The minimum and maximum $\beta$ values.
\item $\beta_k$ : The current KL penalisation weight.
\item $a_{lr_k}$ : The current actor learning rate.
\end{itemize}
\end{enumerate}

\If{$d > 2 * d_{targ}$}
\State $\beta_{k+1} \leftarrow \min(\beta_{max}, \beta_k * 1.5)$
\If{$\beta_{k} > 0.85 * \beta_{max}$}
\State $a_{lr_{k+1}} \leftarrow \frac{a_{lr_k}}{1.5} $
\EndIf
\ElsIf{$d < \frac{d_{targ}}{2}$}
\State $\beta_{k+1} \leftarrow \max(\beta_{min}, \frac{\beta_k}{1.5})$
\If{$\beta_k < 1.15 * \beta_{min}$}
\State $a_{lr_{k+1}} \leftarrow a_{lr_k} * 1.5 $
\EndIf
\EndIf
\end{algorithmic}
\end{algorithm}

\begin{algorithm}[H]
\caption{Update from $\psi_k$ to $\psi_{k+1}$}
\label{alg:critic_math_update}
\begin{algorithmic}[1]
\State \textbf{update critic parameters}({$H_{k-CRB},\ldots,H_{k}$, $\mb{D}$,$\psi_k$,$hp_k$})
\State \textbf{Inputs}: 
\begin{enumerate}[leftmargin=1.5cm,labelsep=0cm,align=left,label={[\arabic*]}]
\item $H_{k-CRB},\ldots,H_{k}$ : The $crb+1$ last sets of trajectories of length $L$. 
\item $\psi_{k}$ : The critic's parameters.
\item $hp_k$ : In this procedure, we use as hyper-parameters:
\begin{itemize}[label={\tiny\raisebox{1ex}{\textbullet}}]
\item $e_{critic}$ : The number of epochs per critic update.
\item $cmb$ : The mini-batch size used for computing the critic's gradient.
\item $T$ : A hyper-parameter of our gradient estimate.
\item $crb$ : The replay buffer size.
\item $c_{lr}$ : The critic learning rate.
\item $L'$ : The number of time-steps per trajectory used for computing gradients.
\item $c_{v_{k}}$ and $c_{z_{k}}$ : The last ADAM moments computed at iteration $k-1$.
\end{itemize}
\end{enumerate}
\State \textbf{Output}: 
\begin{enumerate}[leftmargin=1.5cm,labelsep=0cm,align=left,label={[\arabic*]}]
\item $\psi_{k+1}$ : The updated critic parameters.
\end{enumerate}

\State $m \leftarrow 0$
\State $c'_{v} \leftarrow c_{v_k}$
\State $c'_z \leftarrow c_{z_k}$
\State $\mathcal{T}_{(i,t)} = [[i,t*T],\ldots,[i,max((t+1)*T-1,L')]]~\forall (i, t) \in ([B*(k-CRB),\ldots,B*(k+1)-1]*[0,\ldots,\lfloor \frac{L'}{T} \rfloor])$ 
\State $\mathcal{BT} = [(B*(k-crb)+i, j) \forall (i,j) \in ([0,\ldots,B*(crb+1)-1]*[0,\ldots,\lfloor \frac{L'}{T} \rfloor])]$ 
\State $\psi' \leftarrow \psi_k$
\State $e_{iter} \leftarrow e_{critic} * \lceil \frac{|\mathcal{BT}|}{cmb * T} \rceil$
\State $\mathcal{S} \leftarrow \emptyset$
\While {$m < e_{iter}$}
\State $p \leftarrow 0$, $\mathcal{Y} \leftarrow \emptyset$
\While{$p < cmb \wedge \mathcal{BT} \setminus \mathcal{S} \neq \emptyset$}
\State $(i_{cur}, t_{cur}) \sim \mathcal{BT} \setminus \mathcal{S}$
\State $\mathcal{S} \leftarrow \mathcal{S} \cup (i_{cur}, t_{cur})$
\State $\mathcal{Y} \leftarrow \mathcal{Y} \cup \mathcal{T}_{i_{cur},t_{cur}}$
\State $p \leftarrow p + 1$
\EndWhile
\If {$\mathcal{BT} \setminus \mathcal{S} = \emptyset$}
\State $\mathcal{S} \leftarrow \emptyset$
\EndIf

\State $\mathcal{L}_{sur}(\psi) = \avsum_{(i,t) \in \mathcal{Y}}^{}{(c_{\psi}(h^i_t) - D^i_t)^2}$

\State $\nabla_{\psi}\mathcal{L}_{sur}(\psi', \mathcal{Y}) = \textbf{compute gradients}(\mathcal{L}_{sur}, \mathcal{Y}, \psi')$

\State $c'_{lr} = c_lr * \frac{\sqrt{1-\omega^{k*e_{iter}+m}_2}}{1-\omega^{k*e_{iter}+m}_1}$
\State $c'_z = \omega_1 * c'_z + (1-\omega_1)*\nabla_{\psi}\mathcal{L}_{sur}(\psi', \mathcal{Y})$
\State $c'_v = \omega_2 * c'_v + (1-\omega_2)*\nabla_{\psi}\mathcal{L}_{sur}(\psi', \mathcal{Y})\odot\nabla_{\psi}\mathcal{L}_{sur}(\psi', \mathcal{Y})$
\State $\psi' \leftarrow \psi' - \frac{c'_{lr}*c'_z}{\sqrt{c'_v} +\epsilon}$
\State $m \leftarrow m + 1$
\EndWhile
\State $c_{v_{k+1}} \leftarrow c'_v$
\State $c_{z_{k+1}} \leftarrow c'_z$
\State $\psi_{k+1} \leftarrow \psi'$
\State \textbf{Return} $\psi_{k+1}$
\end{algorithmic}
\end{algorithm}

\begin{algorithm}[H]
\caption{Gradient computing with BPTT \cite{bptt}}
\label{alg:compute_gradients}
\begin{algorithmic}[1]
\State \textbf{compute gradients}($\mathcal{L}(\alpha), \mathcal{Z}, \alpha'$, $hp_k$)
\State \textbf{Inputs}: 
\begin{enumerate}[leftmargin=1.5cm,labelsep=0cm,align=left,label={[\arabic*]}]
\item $\mathcal{L}(\alpha)$ : A loss function which is dependent on a function approximate $v(\alpha)_{\cdot,\cdot}$.
\item $\mathcal{Z}$ : The set of pairs $(i,j)$ such that $v(\alpha)_{i,j}$ appears in $\mathcal{L}(\alpha)$.
\Comment We emphasise that from the way $\mathcal{Z}$ is built (see Algorithms~\ref{alg:actor_math_update} and~\ref{alg:critic_math_update}), most of the time $\mathcal{Z}$ contains $x$ batches of $T$
consecutive pairs. Note that very rarely, batches may have fewer than $T$ consecutive
pairs (whenever a batch contains the last pairs of an episode which does not contain
a multiple of $T$ pairs), although, the same gradient descent algorithm can still be applied.
\item $\alpha'$ : The element for which the estimate gradient of $\mathcal{L}(\alpha)$ needs to be evaluated.
\item $hp_k$ : In this procedure, we use as hyper-parameter:
\begin{itemize}[label={\tiny\raisebox{1ex}{\textbullet}}]
    \item $T$ : The number of time-steps for which the gradient can propagate.
\end{itemize}
\end{enumerate}
\State \textbf{Output}: 
\begin{enumerate}[leftmargin=1.5cm,labelsep=0cm,align=left,label={[\arabic*]}]
\item $\nabla_\alpha \mathcal{L} (\alpha')$ : The gradient estimate of the function $\mathcal{L}(\alpha)$ evaluated in $\alpha'$.
\end{enumerate}

\Comment We refer the reader to the source
code which is available on Github (\url{https://github.com/nvecoven/nmd_net}), which is a particular implementation of standard BPTT \cite{bptt}. We note that giving a full tabular version
of the algorithm here would not constitute valuable information to the reader, due to
its complexity/length.
\end{algorithmic}
\end{algorithm}
\subsection{Architecture details}
\label{app:arch}

For conciseness, let us denote by $f_{n}$ a hidden layer of $n$ neurons with activation functions $f$, by $\rightarrow$ a connection between two fully-connected layers and by $\multimap()$ a neuromodulatory connection (as described in Section~\ref{sec:NMN}).

\paragraph{Benchmark 1.} The architectures used for this benchmark were as follows:

\begin{itemize}
    \item RNN : $GRU_{50} \rightarrow ReLU_{20} \rightarrow ReLU_{10} \rightarrow I_1$
    \item NMN : $GRU_{50} \rightarrow ReLU_{20} \multimap (SReLU_{10} \rightarrow I_1)$
\end{itemize}
    
\paragraph{Benchmark 2 and 3.} The architectures used for benchmark 2 and 3 were the same and as follows:

\begin{itemize}
    \item RNN : $GRU_{100} \rightarrow GRU_{75} \rightarrow ReLU_{45} \rightarrow ReLU_{30} \rightarrow ReLU_{10} \rightarrow I_1$
    \item NMN : $GRU_{100} \rightarrow GRU_{75} \rightarrow ReLU_{45} \multimap (ReLU_{30} \rightarrow ReLU_{10} \rightarrow I_1)$
\end{itemize}


\subsection{Hyper-parameter values}
\label{app:hyperparameters}
\begin{table}[H]
    \centering
    \begin{tabular}{||c|c||}
    \hline
         $B$ & $50$ \\ \hline
         $\lambda$ & $0.98$ \\ \hline
         $\gamma$ & $0.998$\\ \hline
         $\beta_0$ & $1$ \\ \hline 
         $\beta_{min}$ & $1/30$ \\ \hline 
         $\beta_{max}$ & $30$\\ \hline
         $d_{targ}$ & $0.003$ \\ \hline 
         $a_{lr_0}$ & $2*10^{-4}$\\ \hline
         $\omega_1$ & $0.9$ \\ \hline 
         $\omega_2$ & $0.999$ \\ \hline 
         $\epsilon$ & $10^{-8}$\\ \hline
         $e_{actor}$ & $20$ \\ \hline
         $crb$ & $2$ \\ \hline
         $cmb$ & $25$\\ \hline
         $c_{lr}$ & $6*10^{-3}$ \\ \hline 
         $T$ & $200$ \\ \hline
         $e_{critic}$ & $10$\\ \hline
         $c_{v_0}, c_{z_0}, a_{v_0}, a_{z_0}$ & $0$ \\ \hline
         $\eta$ & $50$ \\ \hline
    \end{tabular}
    \caption{Value of the hyper-parameters that are kept constant for every benchmark in this paper.}
    \label{tab:hyperparams}
\end{table}

\newpage
\subsection{Bayes optimal policy for benchmark 1}
\label{app:opt}

A Bayes optimal policy is a policy that maximises the expected sum of rewards it obtains when playing an MDP drawn from a known distribution $\mathcal{D}$. That is, a Bayes optimal policy $\pi^*_{bayes}$ belongs to the following set:

\[\pi^*_{bayes} \in \underset{\pi \in \Pi}{\arg\max}\underset{\underset{ {x_{\cdot}}\sim P_{\mathcal{M}}\mbox{\tiny{(\raisebox{0.6ex}{.},\raisebox{0.6ex}{.})}}}{\underset{ {a_{\cdot}}\sim\pi\mbox{\tiny{(\raisebox{0.6ex}{.})}}}{\underset{{x_0} \sim P_{x_0}}{{\mathcal{M}} \sim \mathcal{D}}}}}{\mathbb{E}} R^{\pi}_{\mathcal{M}}\quad,\] 
with $P_{\mathcal{M}}$ being the state-transition function of the MDP $\mathcal{M}$ and $R^\pi_{\mathcal{M}}$ the discounted sum of reward obtained when playing policy $\pi$ on $\mathcal{M}$.

\noindent In the first benchmark, the MDPs only differ by a bias, which we denote $\alpha$. Drawing an MDP according to $\mathcal{D}$ amounts to draw a value of $\alpha$ according to  a uniform distribution of $\alpha$ over $[-\alpha_{max}, \alpha_{max}]$, denoted by $\mathbb{U}_\alpha$,  and to  determine the transition function and the reward function that correspond to this value. Therefore, we can write the previous equation as:  
\begin{align}
    \pi^*_{bayes} & \in \underset{\pi \in \Pi}{\arg\max}\underset{\underset{ {x_{\cdot}}\sim P_{\mathcal{M}(\alpha)}\mbox{\tiny{(\raisebox{0.6ex}{.},\raisebox{0.6ex}{.})}}}{\underset{ {a_{\cdot}}\sim\pi\mbox{\tiny{(\raisebox{0.6ex}{.})}}}{\underset{{x_0} \sim P_{x_0}}{{\alpha} \sim \mathbb{U}_\alpha}}}}{\mathbb{E}} R^{\pi}_{\mathcal{M}}\quad, \nonumber
\end{align}
with $\mathcal{M}(\alpha)$ being a function giving as output the MDP corresponding to $\alpha$ and $\Pi$ the set of all possible policies.

We now prove the following theorem.

\begin{theorem} \label{theorem:1}
The policy that selects:
\begin{enumerate}
    \item at time-step $t=0$ the action $a_{0} = x_{0} + \frac{\gamma * (\alpha_{max}+4.5)}{1+\gamma}$
    \item at  time-step $t=1$ 
    \begin{enumerate}
        \item[a)] if $r_{0} = 10$, the action  $a_{1} = x_{1} + a_{0} - x_{0}$ 
        \item[b)] else if $|r_{0}| > \alpha_{max} - (a_{0}-x_{0}) \quad \wedge \quad a_{0}-x_{0} > 0$, the action $a_{1} = a_{0} + r_{0}$
        \item[c)] else if  $|r_{0}| > \alpha_{max} - (x_{0} - a_{0}) \quad \wedge \quad a_{0}-x_{0} < 0$, the action  $a_{1} = a_{0} - r_{0}$ 
        \item[d)] and otherwise the action $a_{1} = a_{0} + r_{0} + 1$
    \end{enumerate}
    \item for the remaining time-steps:
    \begin{enumerate}
    \item[a)] if  $r_{0} = 10$, the action  $a_{t} = x_{t} + a_{0} - x_{0}$
    \item[b)] else if $r_{1} = 10$, the action     $a_{t} = x_{t} + a_{1} - x_{1}$
    \item[c)] and otherwise the action $a_t=x_t + i_t $ where $i_t$ is the unique element of the set   $\{a_{0} - x_{0} + r_{0}; a_{0} - x_{0} - r_{0}\} \cap \{a_{1} - x_{1} + r_{1}; a_{1} - x_{1} - r_{1}\}$
    \end{enumerate}
\end{enumerate}
is Bayes optimal for benchmark 1 defined in SubSection~\ref{subsec:benchmarks}. 

\end{theorem}

\begin{proof}
Let us denote by  $\pi^*_{theorem1}$ the policy described in this theorem. 
To prove this theorem, we first prove that in the set of all possible policies $\Pi$ there are no policy $\pi$ which leads to a higher value of
\begin{equation} \label{eq:2StepProblem}
  \underset{\underset{\underset{{x_\cdot} \sim P_{\mathcal{M}}\mbox{\tiny{(\raisebox{0.6ex}{.},\raisebox{0.6ex}{.})}}}{{a_{\cdot}} \sim \pi\mbox{\tiny{(\raisebox{0.6ex}{.})}}}}{\underset{{x_0} \sim P_{x_0}}{\mathcal{M} \sim \mathcal{D}}}}{\mathbb{E}} (r_{0} + \gamma * r_{1})
\end{equation}
\noindent than $\pi^*_{theorem1}$. Or equivalently:

\begin{equation}\label{eq:t_to_2}\underset{\underset{\underset{{x_\cdot} \sim P_{\mathcal{M}}\mbox{\tiny{(\raisebox{0.6ex}{.},\raisebox{0.6ex}{.})}}}{{a_{\cdot}} \sim \pi^*_{theorem1}\mbox{\tiny{(\raisebox{0.6ex}{.})}}}}{\underset{{x_0} \sim P_{x_0}}{\mathcal{M} \sim \mathcal{D}}}}{\mathbb{E}} (r_{0} + \gamma * r_{1}) \geq \underset{\underset{\underset{{x_\cdot} \sim P_{\mathcal{M}}\mbox{\tiny{(\raisebox{0.6ex}{.},\raisebox{0.6ex}{.})}}}{{a_{\cdot}} \sim \pi\mbox{\tiny{(\raisebox{0.6ex}{.})}}}}{\underset{{x_0} \sim P_{x_0}}{\mathcal{M} \sim \mathcal{D}}}}{\mathbb{E}} (r_{0} + \gamma * r_{1}) ~\forall \pi \in \Pi  \quad.\end{equation}

\noindent  Afterwards, we  prove  that  the policy  $\pi^*_{theorem1}$,
generates for  each time-step $t \geq  2$ a reward equal  to $R_{max}$
which is the maximum reward achievable, or written alternatively as:
    
    \begin{equation}\label{eq:t_to_inf}\underset{\underset{\underset{{x_\cdot} \sim P_{\mathcal{M}}\mbox{\tiny{(\raisebox{0.6ex}{.},\raisebox{0.6ex}{.})}}}{{a_{\cdot}} \sim \pi^*_{bayes}\mbox{\tiny{(\raisebox{0.6ex}{.})}}}}{\underset{{x_0} \sim P_{x_0}}{\mathcal{M} \sim \mathcal{D}}}}{\mathbb{E}} (\sum_{t=2}^\infty \gamma^t * r_{t}) = \sum_{t=2}^\infty \gamma^t * R_{max} \geq  \underset{\underset{\underset{{x_\cdot} \sim P_{\mathcal{M}}\mbox{\tiny{(\raisebox{0.6ex}{.},\raisebox{0.6ex}{.})}}}{{a_{\cdot}} \sim \pi\mbox{\tiny{(\raisebox{0.6ex}{.})}}}}{\underset{{x_0} \sim P_{x_0}}{\mathcal{M} \sim \mathcal{D}}}}{\mathbb{E}} (\sum_{t=2}^\infty \gamma^t * r_{t})~\forall \pi \in \Pi\quad.\end{equation}
    
\noindent 
By merging (\ref{eq:t_to_2}) and (\ref{eq:t_to_inf}), we have that 
    \[\underset{\underset{\underset{{x_\cdot} \sim P_{\mathcal{M}}\mbox{\tiny{(\raisebox{0.6ex}{.},\raisebox{0.6ex}{.})}}}{{a_{\cdot}} \sim \pi_{theorem1}\mbox{\tiny{(\raisebox{0.6ex}{.})}}}}{\underset{{x_0} \sim P_{x_0}\mbox{\tiny{(\raisebox{0.6ex}{.})}}}{\mathcal{M} \sim \mathcal{D}}}}{\mathbb{E}} (\sum_{t=0}^\infty \gamma^t * r_{t}) \geq \underset{\underset{\underset{{x_\cdot} \sim P_{\mathcal{M}}\mbox{\tiny{(\raisebox{0.6ex}{.},\raisebox{0.6ex}{.})}}}{{a_{\cdot}} \sim \pi\mbox{\tiny{(\raisebox{0.6ex}{.})}}}}{\underset{{x_0} \sim P_{x_0}\mbox{\tiny{(\raisebox{0.6ex}{.})}}}{\mathcal{M} \sim \mathcal{D}}}}{\mathbb{E}} (\sum_{t=0}^\infty \gamma^t * r_{t})~\forall \pi \in \Pi\]
    \noindent which proves the theorem.

    \paragraph{{ $\triangleright$ \it Part 1}.} Let us now prove inequality (\ref{eq:t_to_2}). The first thing to notice is that for a policy to maximise  expression (\ref{eq:2StepProblem}), it only needs to satisfy two conditions for all $x_0$. The first one: to select an action $a_1$, which knowing the value of $(x_0,a_0,r_0,x_1)$, maximises the expected value of $r_1$. We denote by $V_1(x_0,a_0,r_0,x_1)$ the maximum expected value of $r_1$ that can be obtained knowing the value of $(x_0,a_0,r_0,x_1)$. The second one: to select an action $a_0$ knowing the value of $x_0$ that maximises the expected value of the sum $r_0 + \gamma V_1(x_0,a_0,r_0,x_1)$. We now show that the policy $\pi_{theorem1}$ satisfies these two conditions.

    Let us start with the first condition that we check by analysing four cases, which correspond to the four cases a), b), c), d) of policy $\pi_{theorem1}$ for time step $t=1$.
\begin{enumerate}
\item[a)]  If $r_0 = 10$, the maximum reward that can be obtained,  we are in a context where  $a_0$ belongs to the target interval. It is easy to see that, by playing $a_1 = x_1 + a_0 - x_0$, we will obtain $r_1$ equal to 10. This shows that in case a) for time step $t=1$,  $\pi_{theorem1}$ maximises this expected value of $r_1$. 
\item[b)]  If   $|r_{0}| > \alpha_{max} - (a_{0}-x_{0}) \quad \wedge \quad a_{0}-x_{0} > 0$ and  $r_0 \ne 10$ it is easy  to see  that the value of $\alpha$ to which the MDP corresponds can be inferred from $(x_0,a_0,r_0)$ and that the action $a_{1} = a_{0} + r_{0}$ will fall in the middle of the target interval, leading to a reward of 10. Hence, in this case also, the policy $\pi_{theorem1}$ maximises the expected value of $r_1$.
\item[c)] If  $|r_{0}| > \alpha_{max} - (x_{0} - a_{0}) \quad \wedge \quad a_{0}-x_{0} < 0$ and $r_0 \ne 10$, we are also in a context where the value of $\alpha$ can be inferred directly from $(x_0,a_0,r_0)$ and the action $a_1=a_{0} - r_{0}$ targets the centre of the target interval, leading to a reward of $10$. Here again, $\pi_{theorem1}$ maximises the expected value of $r_1$.
\item[d)] When none of the three previous conditions is satisfied, $a$ is not satisfied and so $x_1 = x_0$,  we need to consider two cases: $(a_0 -x_0) \ge 0$ and $(a_0 -x_0)  < 0$. Let us first start with  $(a_0 -x_0) \ge 0$. In such a context, $\alpha  \in \{ a_{0} - x_{0} + r_{0}; a_{0} - x_{0} - r_{0}\} = \{a_{0}-x_{0} - |a_{0}-x_{0}-\alpha |,a_{0}-x_{0} + |a_{0}-x_{0}-\alpha|\}$ and where:
    \begin{enumerate}
        \item [1)]$P(\alpha = a_{0} - x_{0} - |a_{0}-x_{0}-\alpha| | x_0,a_0,r_0,x_1) = 0.5$
        \item [2)]$P(\alpha = a_{0} - x_{0} + |a_{0}-x_{0}-\alpha| | x_0,a_0, r_0, x_1) = 0.5 \quad .$
    \end{enumerate}
    Let us now determine  the action $a_1$ that  maximises $\hat{r}_1$,  the expected value of $r_1$ according to $P(\alpha| x_0,a_0,r_0,x_1)$. Five cases, represented on Figure~\ref{fig:demo}, have to be considered: 
    \begin{enumerate}
    \item [1)] $a_{1}  < a_{0} - |a_{0}-x_{0}-\alpha|-1$. Here  $ \hat{r}_{1} =  a_{1}-a_{0} $ \noindent
and the  maximum of $\hat{r}_1$ is equal to $ -|a_{0}-x_{0}-\alpha|-1$.

     \item [2)] $a_{1}  \in [a_{0}- |a_{0}-x_{0}-\alpha|-1,a_{0} - |a_{0}-x_{0}-\alpha|+1]$. Here we have $\hat{r}_{1} = \frac{1}{2} (10 + a_0 - |a_0 - x_0 - \alpha | - a_1)$
whose maximum over the interval is   $5.5 - | a_0 - x_0 - \alpha | $ which is reached for $a_1 = a_{0} + |a_{0}-x_{0}-\alpha|-1$.

        \item [3)] $a_{1}  \in [a_{0} - |a_{0}-x_{0}-\alpha|+1,a_{0} + |a_{0}-x_{0}-\alpha|-1]$. In this case $\hat{r}_{1} = -|a_{0}-x_{0}-\alpha|$ and is independent from $a_1$.

\item [4)] $a_{1}  \in [a_{0} + |a_{0}-x_{0}-\alpha|-1,a_{0} + |a_{0}-x_{0}-\alpha|+1]$. The expected reward is 
          $\hat{r}_{1} = \frac{1}{2} (10 + a_0 - |a_0 - x_0 - \alpha | - a_1)$
whose maximum over the interval is  $5.5 - | a_0 - x_0 - \alpha | $ which is reached for $a_1 = a_{0} + |a_{0}-x_{0}-\alpha|+1$.  

\item [5)] $a_{1}  > a_{0} + |a_{0}-x_{0} - \alpha| + 1$. In this case the expected reward is $\hat{r}_{1} =  a_{0}-a_{1}$
and the maximum of $\hat{r}_1$ is equal to $ -|a_{0}-x_{0}-\alpha|-1$.

    \end{enumerate}
    
    \begin{figure}[H]
        \centering
        \definecolor{imblue}{HTML}{44BDDA}
\begin{tikzpicture}[scale=1, minimum size=1em]
\draw[white, line width=1pt] (-6,0) rectangle (7.75,4); 
\begin{scope}
\coordinate (bar) at (0.5, 3.0);
\coordinate (bar_l) at ($(bar)+(-5.6,-0.1)$);
\coordinate (bar_r) at ($(bar)+(5.6,0.1)$);
\coordinate (alpha) at (2.5, 0.1);
\draw[-, line width = 0.8pt, opacity=.8, black] ($(bar)+(alpha)+(0.0,1.2)$) -- ($(bar)+(alpha)+(0.0,-0.4)$);
\draw[dashed, line width = 0.8pt, opacity=.6, black] ($(bar)+(alpha)+(0.4,1.2)$) -- ($(bar)+(alpha)+(0.4,-0.4)$);
\draw[dashed, line width = 0.8pt, opacity=.6, black] ($(bar)+(alpha)+(-0.4,1.2)$) -- ($(bar)+(alpha)+(-0.4,-0.4)$);
\coordinate (alpha2) at (-2.5, 0.1);
\draw[-, line width = 0.8pt, opacity=.8, black] ($(bar)+(alpha2)+(0.0,1.2)$) -- ($(bar)+(alpha2)+(0.0,-0.4)$);
\draw[dashed, line width = 0.8pt, opacity=.6, black] ($(bar)+(alpha2)+(0.4,1.2)$) -- ($(bar)+(alpha2)+(0.4,-0.4)$);
\draw[dashed, line width = 0.8pt, opacity=.6, black] ($(bar)+(alpha2)+(-0.4,1.2)$) -- ($(bar)+(alpha2)+(-0.4,-0.4)$);
\coordinate (state) at (-4.2, 0.1);
\draw[->, line width = 0.8pt, opacity=.8, imblue] ($(bar)+(state)+(0.0,0.7)$) -- ($(bar)+(state)$);
\coordinate (action) at (-1.6, 0.1);
\draw[-, line width = 0.8pt, opacity=.8] ($(bar)+(0.0,0.5)$) -- ($(bar)+(0.0,0.1)$);
\draw[-, dashed, line width = 0.8pt, opacity=.65] ($(bar)+(0.0,-0.1)$) -- ($(bar)+(0.0,-0.9)$);

\node[] at ($(bar)+(0.0,0.7)$) {\footnotesize $a_0$};

\draw[->, color = imblue, line width = 1.5] ($(bar)+(0.0,-0.9)$) -- ($(state)+(bar)+(0.0,-1.)$) node[midway, below]{\footnotesize \textcolor{imblue}{$x_0$}};
\draw[<->, color = black, dashed, line width = 1.5] ($(bar)+(0.0,-0.3)$) -- ($(bar)+(alpha)+(0.0,-0.4)$)node[midway, below]{\footnotesize \textcolor{black}{$|r_0|$}};
\draw[<->, color = black, dashed, line width = 1.5] ($(bar)+(0.0,-0.3)$) -- ($(bar)+(alpha2)+(0.0,-0.4)$)node[midway, below]{\footnotesize \textcolor{black}{$|r_0|$}};

\draw[->, color = imblue, line width = 1.5] ($(bar)+(alpha2)+(0.0,0.4)$) -- ($(state)+(bar)+(0.0,0.4)$) node[midway, above]{\footnotesize \textcolor{imblue}{$\alpha$}};

\draw[<->, color = red, line width = 1.5] ($(bar_l)+(0.0,-1.4)$) -- ($(bar)+(alpha2)+(-0.4,-1.6)$) node[midway, below]{\footnotesize \textcolor{red}{$1$}};
\draw[<->, color = red, line width = 1.5] ($(alpha2)+(bar)+(-0.4,-1.6)$) -- ($(bar)+(alpha2)+(0.4,-1.6)$) node[midway, below]{\footnotesize \textcolor{red}{$2$}};
\draw[<->, color = red, line width = 1.5] ($(alpha2)+(bar)+(0.4,-1.6)$) -- ($(bar)+(alpha)+(-0.4,-1.6)$) node[midway, below]{\footnotesize \textcolor{red}{$3$}};
\draw[<->, color = red, line width = 1.5] ($(alpha)+(bar)+(-0.4,-1.6)$) -- ($(bar)+(alpha)+(0.4,-1.6)$) node[midway, below]{\footnotesize \textcolor{red}{$4$}};
\draw[<->, color = red, line width = 1.5] ($(alpha)+(bar)+(0.4,-1.6)$) -- ($(bar_r)+(0.0,-1.6)$) node[midway, below]{\footnotesize \textcolor{red}{$5$}};
\draw[black, line width=0.5pt] (bar_l) rectangle (bar_r); 

\end{scope}
\end{tikzpicture}

        \caption{Graphical representation of the $5$ different cases when playing $a_1$.}
        \label{fig:demo}
    \end{figure}
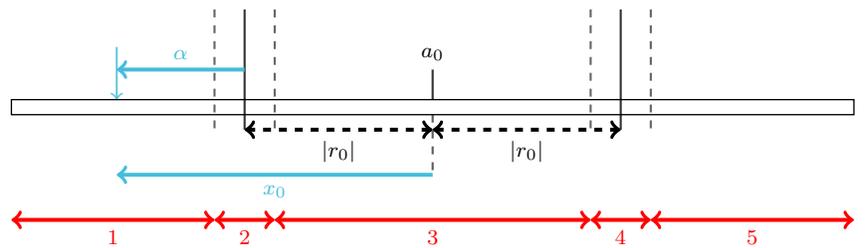
    From 1), 2), 3), 4) and 5) one can see that, given the conditions considered here,  an optimal policy can either play $a_{1} = a_{0} + |a_{0}-x_{0}-\alpha|-1$ or $a_{1} = a_{0}-|a_{0}-x_{0}-\alpha|+1$.
    In the following we will fix $a_1$ to $a_{0} + |a_{0}-x_{0}-\alpha|+1$ when $a_0 - x_0 \geq 0$. Let us also observe that the expected value of $r_1$ is equal to $5.5 - |a_{0} - x_{0} -\alpha|$. Up to now in this item d), we have only considered the case where $(a_{0}-x_{0}) > 0$. When $(a_{0}-x_{0}) \leq 0$, using the same reasoning we reach the exact same expression for  the optimal action to be played and  for  the maximum expected return of $r_1$. This is due to the symmetry that exists between both cases.
Since $\pi_{theorem1}$  plays  the  action $a_{1} =  a_{0} + r_{0} + 1 = a_{0}-|a_{0}-x_{0}-\alpha|+1$ in the case d) at time step $1$, it is straightforward to conclude that, in this case, it also plays an action that maximises the expected value of $r_1$. 
\end{enumerate}

Now that the first condition for $\pi_{theorem1}$ to maximise  expression (\ref{eq:2StepProblem}) has been proved, let us turn our attention to the second one. To this end, we will compute for each  $x_0 \in \mathcal{X}$, the action $a_0 \in \mathcal{A}$ that maximises: 
\begin{equation}\label{eq:expectation}\underset{\underset{x_1 \sim P_{\mathcal{M}(\alpha)}(x_0,a_0)}{\alpha \sim \mathbb{U}_\alpha}}{\mathbb{E}} (r_{0} + \gamma * V_1(x_0, a_0, r_0, x_1))\end{equation}
\noindent and show that this action coincide with the action taken by $\pi_{theorem1}$ for time step $t=0$.
First let us observe that for  this optimisation problem, one can reduce the search space $\mathcal{A}$ to $[x_{0}-\alpha_{max}+1, x_{0}+\alpha_{max}-1] \subset \mathcal{A}$.  Indeed, an action $a_0$ that does not belong to this latter interval would not give more information about $\alpha$ than playing $a_{0} = x_{0}-\alpha_{max}+1$ or $x_{0}+\alpha_{max}-1$ and lead to a worse expected $r_0$. This reduction of the search space will be exploited in the developments that follow.

However, we should first remember that  $\mathbb{U}_\alpha = \mathbb{U}[-\alpha_{max},\alpha_{max}]$  and that the function $V_1(x_0, a_0, r_0, x_1)$ can be written as follows: 
\begin{enumerate}
    \item if $r_0 = 10$, $V_1$ is equal to $R_{max} = 10$ 
    \item else if $|r_0| > \alpha_{max} - (a_0 - x_0)\quad\wedge\quad a_0 - x_0 > 0$ and $r_0 \neq 10$, then $V_1$ is equal to $R_{max} = 10$
    \item else if $|r_0| > \alpha_{max} - (x_{0} - a_{0}) \quad\wedge\quad a_0 - x_0 < 0$ and $r_0 \neq 10$, then $V_1$ is equal to $R_{max} = 10$
    \item and otherwise  $V_1$ is equal to $5.5 - |a_0-x_0-\alpha|$.
\end{enumerate}
We note that the value of $V_1(x_0, a_0, r_0, x_1)$ does not depend on the state $x_1$, which allows us to rewrite expression  (\ref{eq:expectation})  as follows: 
\begin{equation}\label{eq:new_expectation}\underset{\alpha\sim\mathbb{U}_\alpha}{\mathbb{E}} (r_{0} + \gamma * V_1(x_0, a_0, r_0, x_1))\end{equation}

\noindent and since the expectation is a linear operator:

\begin{equation}\label{eq:exp_sum}(\ref{eq:new_expectation}) = \underset{\alpha\sim\mathbb{U}_\alpha}{\mathbb{E}} (r_{0} ) + \gamma * \underset{\alpha\sim\mathbb{U}_\alpha}{\mathbb{E}} (V_1(x_0, a_0, r_0, x_1)) \quad . \end{equation}

\noindent Let us now focus on the second term of this sum:

\begin{equation}\label{eq:exp_v1}
    \underset{\alpha\sim\mathbb{U}_\alpha}{\mathbb{E}} (V_1(x_0, a_0, r_0, x_1)) \quad .
\end{equation}

\noindent We note that when $a_0 - x_0 \geq 0$ the function $V_1$ can be rewritten under the following form:
\begin{enumerate}
    \item if $\alpha \in [-\alpha_{max}, 2*(a_0 - x_0) -\alpha_{max}[$, $V_1$ is equal to 10
    \item else if $\alpha \in [2 * (a_{0}-x_0)-\alpha_{max},a_0-x_0-1]$, $v_1$ is equal to $5.5 + \alpha - (a_0 - x_0)$
    \item else if $\alpha \in [a_0-x_0-1, a_0 - x_0 +1]$, $V_1$ is equal to $10$
    \item else if $\alpha \in ]a_0-x_0+1,\alpha_{max}]$, $V_1$ is equal to $5.5 - \alpha + (a_0 - x_0)$.
\end{enumerate}

From here, we can compute  the value of expression (\ref{eq:exp_v1}) when  $a_0 - x_0 \geq 0$. We note that due to the symmetry that exists between the case $a_0 - x_0 \geq 0$ and $a_0 - x_0 \leq 0$, expression (\ref{eq:exp_v1}) will have the same value for both cases. 
Since we have:
\[(\ref{eq:exp_v1}) = \int_{-\infty}^{\infty} V_1 * p_{\alpha} * d\alpha\]
\noindent where $p_{\alpha}$ is the probability density function of $\alpha$, we can write:

\begin{align}
(\ref{eq:exp_v1}) &= \int_{-\alpha_{max}}^{\alpha_{max}} V_1 * \frac{1}{2*\alpha_{max}} d\alpha \nonumber\\ &= \int_{-\alpha_{max}}^{2*(a_{0}-x_{0})-\alpha_{max}} \frac{10}{2*\alpha_{max}} d\alpha 
 + \int_{2*(a_{0}-x_{0})-\alpha_{max}}^{a_{0}-x_{0}-1} \frac{5.5 + \alpha - (a_{0}-x_{0})}{2*\alpha_{max}} d\alpha \nonumber\\ 
&+ \int_{a_{0}-x_{0}-1}^{a_{0}-x_{0}+1} \frac{10}{2*\alpha_{max}} d\alpha 
+ \int_{a_{0}-x_{0}+1}^{\alpha_{max}} \frac{5.5 - \alpha + (a_{0}-x_{0})}{2*\alpha_{max}} d\alpha \quad . \nonumber
\end{align}

\noindent And thus, by computing the integrals, we have:

\begin{align}\underset{\alpha \sim \mathbb{U}_\alpha}{\mathbb{E}} (V_1) = -\frac{1}{2*\alpha_{max}} (a_{0}-x_{0})^2 + \frac{1}{\alpha_{max}}(\alpha_{max} + 4.5)*(a_{0}-x_{0}) \nonumber\\+ \frac{1}{\alpha_{max}}(5 + 5.5*\alpha_{max} - \frac{\alpha_{max}^2}{2}) \quad . \nonumber \end{align}

\noindent Let us now analyse the first term of the sum in equation (\ref{eq:exp_sum}), namely $\underset{\alpha\sim\mathbb{U}_\alpha}{\mathbb{E}} (r_0)$.

\noindent We have that:
\[ \underset{\alpha\sim\mathbb{U}_\alpha}{\mathbb{E}} (r_0) = \int_{-\infty}^{\infty} (r_{0}|x_{0},a_0,\alpha) * p_{\alpha} * d\alpha\]

\noindent which can be rewritten as:

\[ \underset{\alpha\sim\mathbb{U}_\alpha}{\mathbb{E}} (r_0) = \int_{-\alpha_{max}}^{\alpha_{max}} (r_{0}|x_{0},a_0,\alpha) * \frac{1}{2*\alpha_{max}}d\alpha  \quad . \]

Due to the reduction of the search space, we can  assume that $a_0$  belongs to
$[x_{0}-\alpha_{max}+1,   x_{0}+\alpha_{max}-1]$, we can write:
\begin{align}
\int_{-\alpha_{max}}^{\alpha_{max}} (r_{0}|x_{0}, a_{0}, \alpha) * \frac{1}{2*\alpha_{max}} d\alpha = \int_{-\alpha_{max}}^{a_{0}-x_{0}-1} \frac{\alpha - (a_{0} - x_{0})}{2*\alpha_{max}} d\alpha \nonumber\\ 
+ \int_{a_{0}-x_{0}-1}^{a_{0}-x_{0}+1} \frac{10}{2*\alpha_{max}} d\alpha 
+ \int_{a_{0}-x_{0}+1}^{\alpha_{max}} \frac{(a_{0} - x_{0}) -\alpha}{2*\alpha_{max}} d\alpha \quad . \nonumber
\end{align}
Given that $R_{max} = 10$, we have:

\[ \underset{\alpha\sim\mathbb{U}_\alpha}{\mathbb{E}} (r_0) = \frac{-(a_{0}-x_{0})^2 + 21 - \alpha_{max}^2}{2*\alpha_{max}}\]

\noindent and therefore:

\begin{align}
(\ref{eq:exp_sum})= 
-\frac{1+\gamma}{2*\alpha_{max}}*(a_{0}-x_{0})^2 + \frac{\gamma}{\alpha_{max}}(\alpha_{max} + 4.5)*(a_{0}-x_{0}) \nonumber\\+ \frac{1}{2*\alpha_{max}}(21 - \alpha_{max}^2 + \gamma * (10 + 11*\alpha_{max} -\alpha_{max}^2)) \quad . \nonumber \end{align}

To find the action $a_0$ that maximises (\ref{eq:expectation}), one can differentiate (\ref{eq:exp_sum}) with respect to $a_{0}$:

\begin{align}\frac{d(\ref{eq:exp_sum})}{d(a_{0})} = -\frac{1}{\alpha_{max}}*(1+\gamma)(a_{0}-x_{0}) + \frac{\gamma}{\alpha_{max}}(\alpha_{max}+4.5) \quad . \nonumber \end{align}
This derivative has a single zero value equal to:
\begin{equation}
 a_{0} = \frac{\gamma * (\alpha_{max}+4.5)}{1+\gamma} + x_{0} \quad . \nonumber
\end{equation}
  It can be easily checked that it corresponds to a maximum of expression (\ref{eq:expectation}) and since it also belongs to the reduced search space  $[x_{0}-\alpha_{max}+1,   x_{0}+\alpha_{max}-1]$, it is indeed the solution to our optimisation problem. Since $\pi_{theorem1}$ plays this action at time $t=0$, {\it Part 1} of this proof is now fully completed.

  \paragraph{{ $\triangleright$ \it Part 2}.} Let us now prove that the policy $\pi^*_{theorem1}$ generates for every $t \ge 2$ rewards equal to $R_{max}=10$. We will analyse three different cases, corresponding to the three cases a), b) and c) of policy $\pi_{theorem1}$ for time step $t \ge 2$.
\begin{enumerate}
\item [a)] If $r_0 = 10$, we are in a context where  $a_0$ belong to the target interval. It is straightforward to see that, by playing $a_t = x_t + a_0 - x_0$, the action played by $\pi_{theorem1}$ in this case,   we will get a reward $r_t$ equal to 10.
\item [b)] If $r_1 = 10$ and $ r_0 \ne 10$, one can easily see that  playing action $a_{t} = x_{t} + a_{1} - x_{1}$, the action played by $\pi_{theorem1}$, will always generate rewards equal to $10$.
\item [c)] If $r_0 \ne 10$ and $r_1 \ne 10$, it is possible to deduce from the first action $a_0$ that the MDP played corresponds necessarily to one of these two values for $\alpha$:  $\{ a_{0} - x_{0} + r_{0}; a_{0} - x_{0} - r_{0}\}$. Similarly, from the second action played, one knows that  $\alpha$ must also stand in  $\{ a_{1} - x_{1} + r_{1}; a_{1} - x_{1} - r_{1}\}$. It can be proved that  because $a_{0} \neq a_{1}$ (a property of our policy $\pi_{theorem1}$), the two sets have only one element in common. Indeed if these two sets  had all their elements in common, either this pair of equalities would be valid:
    \begin{align*}
        a_{0} - x_{0} + r_{0} = a_{1} - x_{1} + r_{1} \\
        a_{0} - x_{0} - r_{0} = a_{1} - x_{1} - r_{1} 
    \end{align*}
    or this pair  of equalities would be valid:
        \begin{align*}
        a_{0} - x_{0} + r_{0} =  a_{1} - x_{1} - r_{1} \\
        a_{0} - x_{0} - r_{0} =  a_{1} - x_{1} + r_{1} \quad . 
    \end{align*}
By summing member by member the two equations of the first pair, we have: 
    \begin{align*}
        a_{0} - x_{0} = a_{1} - x_{1} \quad .
    \end{align*}
    Taking into account that $x_{0} = x_{1}$ because none of the two actions yielded a positive reward, it implies that $a_{0} = a_{1}$, which results in a contradiction. It can be shown in a similar way  that another contradiction appears with the second pair.  As a result the intersection of these two sets is unique and equal to $\alpha$. From here, it is straightforward to see that in this case c), the policy $\pi_{theorem1}$ will always  generate rewards equal to $R_{max}$. 
\end{enumerate}

\end{proof}

From Theorem \ref{theorem:1}, one can easily prove the following theorem.
\begin{theorem}
\label{theorem:2}
  The value of expected return of a Bayes optimal policy for benchmark 1 is equal to  $\frac{3*\gamma^2 * (\alpha_{max}+4.5)^2}{2*\alpha_{max}*(1+\gamma)} + \frac{21 + \alpha_{max}^2 + \gamma*(10 + 11*\alpha_{max} - \alpha_{max}^2)}{2*\alpha_{max}} +\frac{\gamma^2}{1-\gamma} * 10$.
\end{theorem}

\begin{proof}
The expected return of  a Bayes optimal policy  can be written as follows:
\begin{equation}\underset{\underset{ {x_{\cdot}}\sim P_{\mathcal{M}}\mbox{\tiny{(\raisebox{0.6ex}{.},\raisebox{0.6ex}{.})}}}{\underset{ {a_{\cdot}}\sim\pi^*_{bayes}\mbox{\tiny{(\raisebox{0.6ex}{.})}}}{\underset{{x_0} \sim P_{x_0}}{\mathcal{M} \sim \mathcal{D}}}}}{\mathbb{E}} \sum_{t=0}^{1} \gamma^t * r_{t} + \underset{\underset{ {x_{\cdot}}\sim P_{\mathcal{M}}\mbox{\tiny{(\raisebox{0.6ex}{.},\raisebox{0.6ex}{.})}}}{\underset{ {a_{\cdot}}\sim\pi^*_{bayes}\mbox{\tiny{(\raisebox{0.6ex}{.})}}}{\underset{{x_0} \sim P_{x_0}}{\mathcal{M} \sim \mathcal{D}}}}}{\mathbb{E}} \sum_{t=2}^{\infty} \gamma^t * r_{t}
\quad . \nonumber
\end{equation}

From the proof of Theorem~\ref{theorem:1}, it is easy to see that:
\begin{enumerate}
    \item $\underset{\underset{ {x_{\cdot}}\sim P_{\mathcal{M}}\mbox{\tiny{(\raisebox{0.6ex}{.},\raisebox{0.6ex}{.})}}}{\underset{ {a_{\cdot}}\sim\pi^*_{bayes}\mbox{\tiny{(\raisebox{0.6ex}{.})}}}{\underset{{x_0} \sim P_{x_0}}{\mathcal{M} \sim \mathcal{D}}}}}{\mathbb{E}} \sum_{t=0}^{1} \gamma^t * r_{t} = \frac{3*\gamma^2 * (\alpha_{max}+4.5)^2}{2*\alpha_{max}*(1+\gamma)} + \frac{21 + \alpha_{max}^2 + \gamma*(10 + 11*\alpha_{max} - \alpha_{max}^2)}{2*\alpha_{max}}$ 
    \item $\underset{\underset{ {x_{\cdot}}\sim P_{\mathcal{M}}\mbox{\tiny{(\raisebox{0.6ex}{.},\raisebox{0.6ex}{.})}}}{\underset{ {a_{\cdot}}\sim\pi^*_{bayes}\mbox{\tiny{(\raisebox{0.6ex}{.})}}}{\underset{{x_0} \sim P_{x_0}}{\mathcal{M} \sim \mathcal{D}}}}}{\mathbb{E}} \sum_{t=2}^{\infty} \gamma^t * r_t = \frac{\gamma^2}{1-\gamma} 10$
\end{enumerate}
which proves Theorem~\ref{theorem:2}.
\end{proof}
\end{document}